\documentclass{article}

\PassOptionsToPackage{numbers, sort&compress}{natbib}
\PassOptionsToPackage{dvipsnames}{xcolor}

\usepackage[preprint]{neurips_2026}

\usepackage[utf8]{inputenc}
\usepackage[T1]{fontenc}
\usepackage{hyperref}
\usepackage{url}
\usepackage{booktabs}
\usepackage{amsfonts}
\usepackage{nicefrac}
\usepackage{microtype}
\usepackage{xcolor}
\usepackage{amsmath, amssymb, amsthm}
\usepackage{algorithm}
\usepackage{algpseudocode}
\usepackage{graphicx}
\usepackage{enumitem}
\usepackage{multirow}
\usepackage{tabularx}
\usepackage{tcolorbox}

\newtheorem{theorem}{Theorem}

\title{LLM-PDESR: Robust PDE Discovery via Subdomain Weighted Residuals and LLM-Guided Symbolic Hypothesis Generation}

\author{%
\begin{tabular}{c}
Jinyang Du$^{1}$
\quad
Hao Ma$^{1}$
\quad
Xiaohu Shi$^{2}$
\quad
Bo Yang$^{1,3}$
\\
Yanchun Liang$^{4}$
\quad
Heow Pueh Lee$^{5}$
\quad
Chunguo Wu$^{1,3,*}$
\\[5pt]
{\footnotesize
$^{1}$College of Computer Science and Technology,
Jilin University, Changchun, China}
\\
{\footnotesize
$^{2}$School of Big Data and Artificial Intelligence,
Guangdong University of Finance and Economics}
\\
{\footnotesize
Guangzhou, China}
\\
{\footnotesize
$^{3}$Key Laboratory of Symbolic Computation and Knowledge Engineering}
\\
{\footnotesize
of Ministry of Education, Jilin University, Changchun, China}
\\
{\footnotesize
$^{4}$School of Computer Science,
Zhuhai College of Science and Technology, Zhuhai, China}
\\
{\footnotesize
$^{5}$Department of Mechanical Engineering,
National University of Singapore, Singapore}
\\[3pt]
{\footnotesize
$^{*}$Corresponding author:
\texttt{wucg@jlu.edu.cn}}
\end{tabular}%
}

\begin{document}

\maketitle

\begin{abstract}
Discovering governing partial differential equations (PDEs) from noisy observational data is a fundamental challenge in scientific machine learning. Traditional symbolic regression (SR) methods often struggle to identify accurate equations within vast combinatorial search spaces, largely due to their inability to incorporate essential domain-specific prior knowledge. Furthermore, reliance on pointwise evaluations and discrete finite differences inherently amplifies high-frequency noise, creating deceptive fitness landscapes that derail the optimization process. To resolve these bottlenecks, we propose LLM-PDESR, a framework that integrates the structural hypothesis generation of Large Language Models (LLMs) with a mathematically rigorous evaluation environment. By employing $\mathcal{C}^4$ continuous quintic splines for robust differentiation and subdomain weighted residuals as natural low-pass filters, our approach effectively mitigates the fitness landscape distortion that plagues existing methods. A Pareto-driven feedback loop then enables the LLM to iteratively refine candidate equations, balancing predictive accuracy with structural parsimony. We evaluate LLM-PDESR on 23 canonical PDEs and five structurally novel equations (including a multivariate system) specifically designed to preclude dataset memorization and test true discovery capabilities. Demonstrating real-world applicability, the framework successfully extracts a consistent structural skeleton for an interpretable 1D dynamical surrogate (1D-CACE) directly from noisy ERA5 reanalysis data. Extensive experiments and out-of-distribution testing confirm that LLM-PDESR significantly outperforms state-of-the-art methodologies in structural recovery, noise resilience, and the avoidance of spurious complexity and equation bloat.
\end{abstract}

\section{Introduction}

A central pursuit in modern computational science is bridging the gap between abundant empirical measurements and the underlying mathematical laws that dictate system dynamics \cite{rudy2017data,long2018pde,raissi2019physics,both2021deepmod,karniadakis2021physics}. While the formalization of this problem through symbolic regression (SR) has spurred significant advances \cite{dong2025recent,shojaee2023transformer,xiang2025graph}, traditional data-driven methodologies ranging from sparse regression \cite{brunton2016discovering,bottmer2022sparse} to genetic programming \cite{zhong2018multifactorial,du2025deep,mei2022explainable} exhibit critical limitations. Constrained by predefined function libraries, these approaches struggle to scale across vast combinatorial hypothesis spaces and lack the architectural flexibility to incorporate domain-specific physical priors \cite{makke2024interpretable}.

Recently, Large Language Models (LLMs) have emerged as highly promising engines for this task \cite{romera2024mathematical,liu2024rl,wang2023scientific,ai4science2023impact,wang2024scimon}. By treating equations as executable structures, LLMs leverage their pre-trained scientific knowledge to propose physically plausible structural hypotheses \cite{shojaee2024llm}. However, extending these LLM-driven paradigms to the discovery of complex, time-varying PDEs is bottlenecked by severe numerical instabilities. To compute the necessary spatial derivatives for evaluating PDE hypotheses, traditional optimization frameworks predominantly rely on standard finite difference (FD) methods and pointwise residual evaluations \cite{teixeira2023finite,leveque1998finite,patidar2005use}. These discrete operations inherently amplify high-frequency measurement noise, distorting the optimization landscape \cite{othmane2022survey,kaheman2022automatic}. When coupled with an LLM, this numerical failure is fatal: it feeds deceptive fitness scores back to the language model, derailing the evolutionary search and preventing it from distinguishing true physical mechanisms from mathematical artifacts.

To resolve this critical bottleneck, we introduce LLM-PDESR (Large Language Model-guided Partial Differential Equation Symbolic Regression), a novel framework that integrates the structural hypothesis generation of LLMs with a mathematically rigorous evaluation environment. First, we replace noise-sensitive FD methods with quintic spline smoothing. By guaranteeing $\mathcal{C}^4$ continuity, quintic splines provide noise-robust spatial derivatives up to the fourth order \cite{alam2025high,erkorkmaz2005quintic}. Second, we transition from pointwise residuals to a subdomain weighted residual(SWR) evaluation. By integrating the PDE residual against smooth test functions over localized sub-domains, the SWR formulation acts as a natural low-pass filter \cite{messenger2025asymptotic}. This mathematically principled smoothing mitigates fitness landscape distortion, ensuring the internal continuous optimizer converges reliably and provides the LLM with accurate fitness signals. Consequently, a Pareto-driven feedback loop efficiently balances equation complexity and predictive accuracy, naturally pruning unphysical bloat \cite{xu2023pareto,liu2025pareto}.

We evaluate LLM-PDESR on a comprehensive benchmark of 23 canonical PDEs, alongside five structurally novel equations explicitly formulated to preclude dataset memorization and validate genuine symbolic discovery capabilities. Demonstrating real-world applicability, our framework successfully extracts a novel coupled atmospheric circulation surrogate (1D-CACE) directly from noisy ERA5 climate data \cite{hersbach2020era5}. Since capturing such intricate multivariate interactions severely degrades or exceeds existing baseline capabilities, we validate the discovered equations through rigorous out-of-distribution (OOD) testing across distinct spatiotemporal domains. This confirms they capture invariant dynamical mechanisms rather than localized statistical artifacts. Extensive empirical comparisons demonstrate that LLM-PDESR significantly outperforms state-of-the-art methods (including SGA-PDE \cite{chen2022symbolic}, DISCOVER \cite{du2024discover}, WSINDy\_PDE \cite{messenger2021weak}, and EqGPT \cite{xu2025generative}) in structural recovery, noise resilience, and the avoidance of spurious equations.

The main contributions of this work are summarized as follows:
\begin{itemize}[leftmargin=*]
    \item We propose LLM-PDESR, an end-to-end framework that integrates the structural hypothesis generation of LLMs with continuous numerical optimization and evolutionary search to automate PDE discovery.
    
    \item We resolve the derivative noise bottleneck by leveraging $\mathcal{C}^4$ continuous quintic splines and SWR evaluations. This approach mitigates fitness landscape distortion, while a Pareto-driven feedback loop optimally balances predictive accuracy with structural parsimony.
    
    \item We introduce a rigorous benchmark comprising 23 canonical PDEs and five structurally novel equations. This benchmark is explicitly designed to preclude dataset memorization and validate genuine symbolic discovery capabilities.
    
   \item We demonstrate real-world applicability by extracting an interpretable 1D dynamical surrogate (1D-CACE) directly from noisy ERA5 climate data. Rigorous spatiotemporal OOD evaluations confirm that the framework captures invariant dynamical mechanisms rather than localized statistical artifacts.
\end{itemize}

\section{Related work}

\subsection{Dictionary-based methods and deep learning surrogates}

Data-driven PDE discovery initially leveraged sparse regression, exemplified by SINDy \cite{brunton2016discovering}, PDE-FIND, and STRidge \cite{rudy2017data}, which identify governing laws by selecting terms from predefined mathematical libraries. To mitigate the acute sensitivity of numerical differentiation to observation noise, integral-based approaches like WSINDy and WSINDy\_PDE \cite{messenger2021weak} were developed, successfully utilizing weak forms as natural low-pass filters. Concurrently, deep learning architectures such as Physics-Informed Neural Networks (PINNs) \cite{raissi2019physics}, Fourier Neural Operators (FNO) \cite{li2020fourier}, and DeepONet \cite{lu2021learning} have demonstrated remarkable predictive capabilities. However, these neural surrogates typically act as black boxes, precluding explicit equation extraction. Moreover, while sparse regression offers interpretability, its reliance on handcrafted dictionaries imposes a rigid structural bottleneck. This rigidity not only triggers an intractable combinatorial explosion when scaling to highly nonlinear and high-dimensional physical systems, but also lacks the compositional flexibility to incorporate complex, domain-specific priors required for discovering novel, free-form dynamics.

\subsection{Evolutionary and neural symbolic regression}

To overcome the structural rigidity of fixed dictionaries, Genetic Programming (GP) tools like Eureqa \cite{schmidt2009distilling} and PySR \cite{cranmer2023interpretable}, alongside hybrid frameworks such as SGA-PDE \cite{chen2022symbolic} and EPDE \cite{maslyaev2021partial}, explore free-form equations via stochastic mutation. However, these methods are often hindered by expression bloat and inefficient convergence within vast combinatorial search spaces. Neural Symbolic Regression (NSR) paradigms address these search inefficiencies by leveraging reinforcement learning and transformer architectures, as demonstrated by DSR \cite{petersen2019deep}, MORL4PDEs \cite{zhang2024morl4pdes}, and DISCOVER \cite{du2024discover}. More recently, EqGPT \cite{xu2025generative} trained a specialized transformer on mathematical corpora to construct an automated generation-evaluation loop. Yet, these autoregressive models operate within purely syntactic spaces lacking explicit physical priors, and predominantly rely on discrete FDs for evaluation. This reliance drastically amplifies high-frequency noise—particularly in high-order or cross-derivatives—distorting the fitness landscape and degrading robustness. In contrast, LLM-PDESR resolves both limitations: it repurposes the LLM as an active reasoning agent bridging syntax and physics via Pareto-driven feedback, while employing a continuous, noise-robust evaluation environment.

\section{The LLM-PDESR framework}
\label{sec:framework}
\begin{figure}[htbp]
    \centering
    \includegraphics[width=\textwidth]{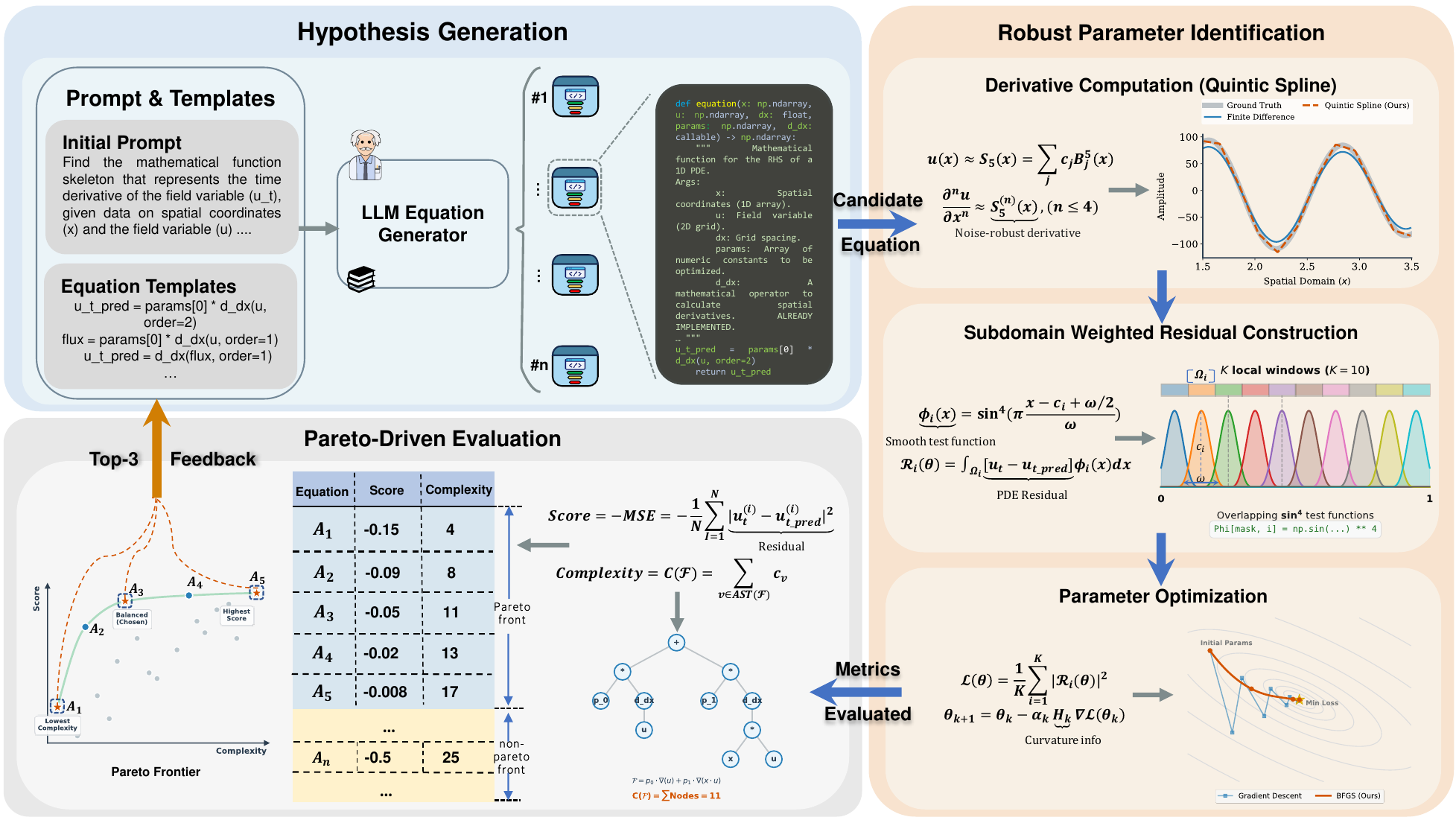} 
    \caption{\textbf{The LLM-PDESR architecture.} The automated pipeline integrates three core modules: (1) LLM-Driven Hypothesis Generation, (2) Robust Parameter Identification via splines, SWR, and BFGS optimization, and (3) Pareto-Driven Evaluation providing dynamic feedback.}
    \label{fig:framework}
\end{figure}

We consider a physical system evolving over a spatial domain $\Omega_x$ and temporal domain $\Omega_t$, characterized by a state field $\mathbf{u}(x,t) \in \mathbb{R}^D$ (e.g., $D=1$ for canonical PDEs like Burgers equation; $D>1$ for coupled fluid dynamics). Given discrete, noisy spatiotemporal observations, our objective is to discover the underlying partial differential operator governing its temporal evolution:
\begin{equation}
    \frac{\partial \mathbf{u}}{\partial t} = \mathcal{N}\left(\mathbf{u}, \frac{\partial \mathbf{u}}{\partial x}, \frac{\partial^2 \mathbf{u}}{\partial x^2}, \dots ; \theta \right)
\end{equation}
where $\mathcal{N}$ represents the unknown structural skeleton to be generated, and $\theta$ denotes the optimizable physical parameters. 

To prevent noise amplification, the target temporal derivative ($\partial \mathbf{u}/\partial t$) is pre-computed via a continuous cubic spline approximation along the time axis.

Traditionally, parameter optimization involves minimizing a pointwise mean squared error (MSE) loss between the observed and predicted temporal derivatives \cite{chicco2021coefficient,qiao2025hybrid}. Evaluating this objective function, however, fundamentally requires approximating the requisite spatial derivatives using discrete FD methods. This dependency introduces a critical evaluation bottleneck: discrete FD inherently amplifies high-frequency observation noise ($\mathcal{O}(\sigma / (\Delta x)^k)$ for $k$-th order derivatives \cite{leveque2007finite}). This numerical instability distorts the fitness landscape, yielding deceptive evaluation scores that derail search algorithms and prevent convergence to the underlying physical laws. 

To resolve this bottleneck, we propose LLM-PDESR. As illustrated in Figure~\ref{fig:framework} and formalized later in Algorithm~\ref{alg:llm_engine}, this automated, closed-loop pipeline addresses PDE discovery through three core modules: LLM-Driven Hypothesis Generation, Robust Parameter Identification, and Pareto-Driven Evaluation.

\begin{algorithm}[tb]
\caption{LLM-PDESR: The End-to-End Discovery Framework}
\label{alg:llm_engine}
\begin{algorithmic}[1]
\Require Observations $\mathcal{D}$, Initial prompt $\mathcal{C}$, Max generations $G$
\Ensure Optimal PDE skeleton $\mathcal{N}^*$ and optimized parameters $\theta^*$
\State \textbf{Initialize:} Feedback context $\mathcal{P}_{fb} \gets \emptyset$, Pareto Frontier $\mathcal{P} \gets \emptyset$

\For{$g = 1$ \textbf{to} $G$}
    \State Sample population of skeletons $\mathcal{S} \sim \text{LLM}(\mathcal{C} \cup \mathcal{P}_{fb})$
    \For{each $\mathcal{N}_j \in \mathcal{S}$}
        \State $\theta_j^*, \text{MSE}_j \gets \text{Numerical\_Evaluation}(\mathcal{N}_j, \mathcal{D})$ \Comment{\textit{See Appx. Alg. \ref{alg:evaluation}}}
        \State Compute AST Complexity $C_j$
    \EndFor
    \State Update global Pareto Frontier $\mathcal{P}$ with evaluated tuples $\{ (\mathcal{N}_j, \theta_j^*, \text{MSE}_j, C_j) \}$
    \State $\mathcal{P}_{fb} \gets \text{Serialize}(\text{Top-3 elites from } \mathcal{P})$ \Comment{Provide LLM feedback}
\EndFor

\State \textit{// Automated parsimony-guided selection within an order-of-magnitude tolerance}
\State $\text{MSE}_{\min} \gets \min_{\mathcal{N} \in \mathcal{P}} \text{MSE}(\mathcal{N})$ 
\State $\mathcal{N}^* \gets \arg\min_{\mathcal{N} \in \mathcal{P}} \{ C(\mathcal{N}) \mid \text{MSE}(\mathcal{N}) \le 10 \times \text{MSE}_{\min} \}$
\State \textbf{Retrieve} $\theta^*$ corresponding to $\mathcal{N}^*$
\State \textbf{return} $\mathcal{N}^*, \theta^*$
\end{algorithmic}
\end{algorithm}

\subsection{LLM-driven hypothesis generation}
Framing the LLM as an automated hypothesis generator, our prompt emulates the empirical discovery pipeline by mapping physical observations (e.g., 1D heat conduction data) to candidate governing equations. To preclude the generation of syntactically or physically invalid structures, this search space is strictly bounded by predefined constraint conditions (e.g., enforcing the spatial derivative operator \texttt{d\_dx(u, order)}). The complete prompts and constraint specifications are detailed in Appendix \ref{Appendix:A}. 

Conditioned on this prompt, the LLM acts as a structural architect, proposing a population of $n$ candidate executable equations (represented as Python Abstract Syntax Trees). Crucially, the LLM generates only the symbolic skeletons, utilizing parameter placeholders (e.g., \texttt{params[0]}) for all numeric coefficients. These $n$ candidate equations are subsequently forwarded to the numerical environment to evaluate their empirical fitness.

\subsection{Robust parameter identification}
The parameter identification module bridges the gap between discrete symbolic generation and continuous observational data through a rigorous three-step numerical pipeline:

\textbf{1. Derivative Computation (Quintic Spline):} To eliminate the aforementioned FD noise bottleneck, we reconstruct the noisy spatial field at each time step $t$ using quintic B-spline smoothing. This enables analytical and noise-robust spatial differentiation:
\begin{equation}
    u(x,t) \approx S_5(x;t) = \sum_{j} c_j(t) B_j^5(x), \quad \frac{\partial^n u}{\partial x^n} \approx S_5^{(n)}(x) \quad (n \le 4)
\end{equation}
where $c_j(t)$ are the spline coefficients and $B_j^5(x)$ denotes the quintic basis functions.

\textbf{2. Subdomain Weighted Residual Construction:} To further fortify the evaluation against localized anomalies, we compute the integral residual $\mathcal{R}_i(t;\theta)$ over $K$ overlapping local sub-domains $\Omega_i$. By integrating against a compactly supported smooth test function $\phi_i(x)$, this formulation inherently acts as a low-pass filter, significantly attenuating high-frequency noise prior to error calculation:
\begin{equation}
    \phi_i(x) = \sin^4\left(\pi \frac{x - c_i + \omega/2}{\omega}\right), \quad \mathcal{R}_i(t;\theta) = \int_{\Omega_i} \left[u_t - u_{t\_\text{pred}}\right] \phi_i(x) dx
\end{equation}
Here, $c_i$ and $\omega$ define the center and width of the local window $\Omega_i$, respectively, and $u_{t\_\text{pred}}$ is the time derivative predicted by the candidate equation parameterized by $\theta$.

\textbf{3. Parameter Optimization via BFGS:} The objective function $\mathcal{L}(\theta)$ is explicitly formulated as the MSE of the integrated residuals. To minimize this non-convex loss, we deploy the quasi-Newton BFGS algorithm \cite{nocedal2006numerical,mannel2024structured}, which iteratively updates the parameters using an approximate inverse Hessian matrix $H_k$:
\begin{equation}
    \mathcal{L}(\theta) = \frac{1}{MK} \sum_{j=1}^M \sum_{i=1}^K | \mathcal{R}_i(t_j; \theta) |^2, \quad \theta_{k+1} = \theta_k - \alpha_k H_k \nabla \mathcal{L}(\theta_k)
\end{equation}
where $M$ and $K$ denote the number of temporal snapshots and spatial sub-domains, respectively; $\alpha_k$ is the step size, and $H_k$ is the approximate inverse Hessian updated via BFGS.

Rigorous theoretical proofs establishing the error bounds of quintic splines over FD, the noise-attenuation properties of the SWR versus point-wise MSE, and the convergence advantages of BFGS in this context are detailed in Appendix \ref{Appendix:B}.

\subsection{Pareto-driven evaluation and automated selection}
Following parameter identification, the optimized candidates are assessed in the Evaluation module. In accordance with Occam's razor, equation discovery is modeled as a multi-objective optimization problem. We map each candidate into a 2D objective space using two strictly defined metrics:
\begin{itemize}[leftmargin=*]
    \item \textbf{Score:} Defined as the negative MSE ($-\text{MSE}$) of the optimized SWR.
    \item \textbf{Complexity:} Quantified as the total number of nodes in the candidate's Abstract Syntax Tree ($C(\mathcal{F}) = \sum_{v \in \text{AST}} c_v$).
\end{itemize}

To guide the LLM's evolutionary search while maintaining context efficiency, we extract three elite solutions from the Pareto frontier $\mathcal{P}$ (highest score, lowest complexity, and balanced) and dynamically inject their serializations into the subsequent prompt. For the final parsimony-guided extraction, we employ a programmatic selection heuristic based on a logarithmic error tolerance. Specifically, we define a subset of highly accurate candidates $\mathcal{P}^* = \{ \mathcal{F} \in \mathcal{P} \mid \text{MSE}(\mathcal{F}) \le 10 \times \text{MSE}_{\min} \}$, representing equations within one order of magnitude of the global minimum error. The algorithm ultimately outputs $\mathcal{F}^* = \arg\min_{\mathcal{F} \in \mathcal{P}^*} C(\mathcal{F})$. This mechanism systematically prunes over-parameterized models while guaranteeing near-optimal predictive accuracy.

\section{Experiments}
\subsection{Experimental setup and baselines}
\textbf{Datasets and Scenarios.} Our evaluation encompasses three distinct regimes: 

(1) \textbf{Benchmark PDEs:} A library of 23 canonical 1D PDEs (e.g., linear, semilinear, Burgers, and high-order nonlinear systems; detailed in Appendix \ref{Appendix:C}). To simulate realistic measurement imperfections, the clean spatial fields $u$ are corrupted with additive Gaussian noise: $\tilde{u} = u + \gamma \sigma_u \epsilon$, where $\sigma_u$ is the standard deviation of $u$, $\epsilon \sim \mathcal{N}(0, 1)$, and $\gamma \in \{0.05, 0.10, 0.20\}$ denotes the noise level. 

(2) \textbf{Structurally Novel Equations:} Five synthetically designed physical systems featuring highly nonlinear spatial terms and unconventional interactions. These are explicitly formulated to preclude dataset memorization and validate genuine symbolic discovery capabilities.

(3) \textbf{Real-World Observations:} Hourly macroscopic wind velocity fields from the ERA5 atmospheric reanalysis dataset \cite{hersbach2020era5}, utilized to extract the interpretable 1D dynamical surrogate (1D-CACE).

\textbf{Baselines and Metrics.} We benchmark our framework against state-of-the-art methodologies spanning evolutionary algorithms, sparse regression, and neural symbolic generation: SGA-PDE \cite{chen2022symbolic}, DISCOVER \cite{du2024discover}, WSINDy\_PDE \cite{messenger2021weak}, and EqGPT \cite{xu2025generative}. Performance is quantitatively assessed using two rigorous metrics: \textit{Structural Success Rate}, strictly defined as the exact symbolic recovery of the true PDE skeleton; and \textit{Relative Coefficient Error}, computed via the $L_2$ norm $\|\hat{\theta} - \theta\|_2 / \|\theta\|_2$ exclusively for successfully identified structures. Detailed experimental configurations and parameter settings for all evaluated methods are provided in Appendix \ref{Appendix:D}.

\subsection{Robustness on benchmark PDEs}
\begin{figure}[htbp]
    \centering
    \includegraphics[width=\textwidth]{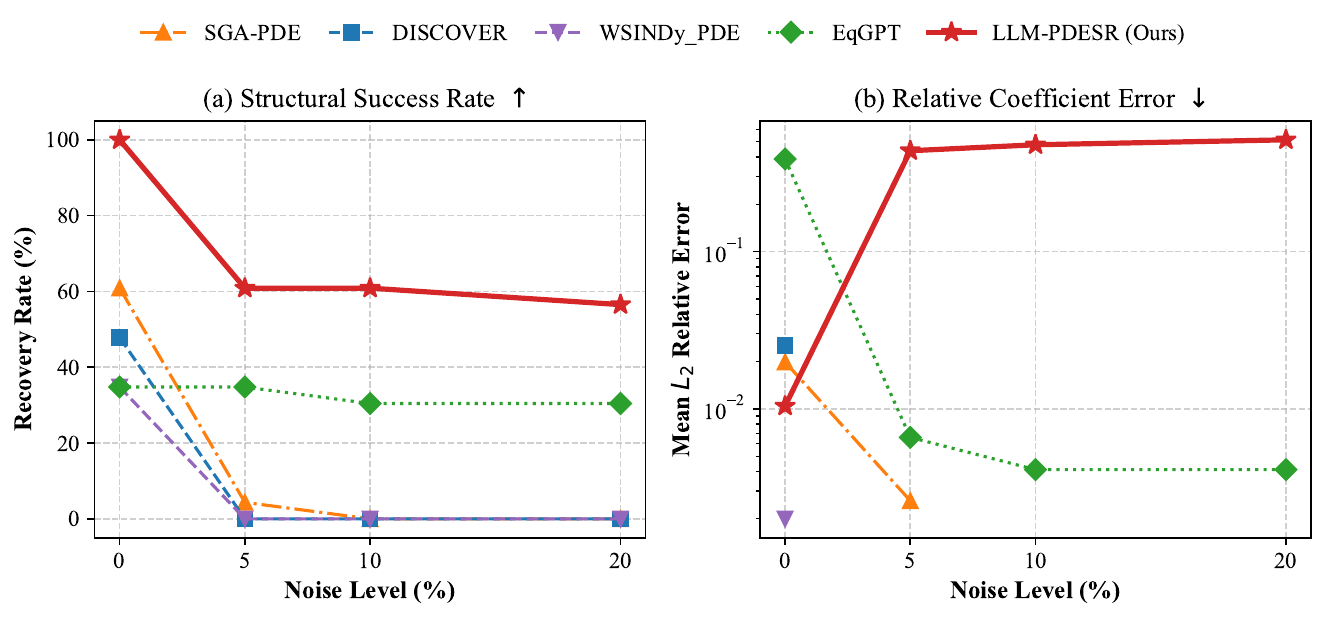} 
    \caption{\textbf{Noise robustness on benchmark PDEs.} Impact of varying observational noise on (a) structural success rate and (b) relative coefficient error.}
    \label{fig:noise_robustness}
\end{figure}

We first evaluate the \textit{Structural Success Rate} (Figure~\ref{fig:noise_robustness}a). LLM-PDESR achieves a 100\% success rate at $0\%$ noise and maintains consistently high robustness as noise increases ($5\%$--$20\%$). Most baselines (e.g., SGA-PDE, DISCOVER, WSINDy\_PDE) suffer severe degradation, while EqGPT retains only partial robustness via explicit data denoising. LLM-PDESR consistently outperforms all competitors, demonstrating the superior noise resilience of our evaluation environment.

Regarding parameter accuracy, Figure~\ref{fig:noise_robustness}b reports the \textit{Relative Coefficient Error} exclusively for successful identifications. At $0\%$ noise, our framework exhibits a low coefficient error. The nominally lower baseline errors in this regime stem from a statistical \textit{survivorship bias}: these methods only converge on the simplest equations. Under noisy conditions, EqGPT exhibits the lowest error due to its external data preprocessing. In contrast, LLM-PDESR ingests raw, noisy observations directly to preserve high-fidelity physical dynamics. Despite omitting heuristic denoising, our framework maintains highly competitive error bounds, confirming that the SWR's intrinsic low-pass filtering effectively isolates true parameters from noise.

\subsection{Structurally novel equations}
Evaluating LLMs exclusively on canonical PDEs inherently risks testing training data memorization rather than genuine symbolic reasoning. To rigorously probe the generalization capabilities of our framework and explicitly mitigate the risk of training data contamination, we designed five structurally novel equations:

\begin{itemize}[leftmargin=*]
    \item \textbf{Topography-Constrained Chemotaxis} ($u_t = 0.1 u_{xx} - 0.5 ( u(1-u) \cos(x) )_x + u(1-u)$): Models biological foraging on undulating terrain with logistic crowding. The nested derivative structure strictly prevents simple additive token generation.
    
    \item \textbf{Morphogenesis with Michaelis-Menten Kinetics} ($u_t = 0.1 u_{xx} - 1.0 u / (0.5 + u) + \exp(-x^2)$): Features a non-polynomial rational fraction and a spatial Gaussian source, a structure typically intractable for traditional polynomial-biased sparse regression.
    
    \item \textbf{Forced Quintic Swift-Hohenberg} ($u_t = 1.5\exp(-0.1x^2)u - 2u_{xx} - u_{xxxx} - u^3 + u^5$): A synthetic, out-of-distribution model inspired by optical microcavity pattern formation. It features a 5th-order nonlinearity and 4th-order spatial derivatives, explicitly designed to break standard physical symmetries.
    
    \item \textbf{Traffic Flow with Spatial Bottleneck} ($u_t = -( u(1-u)(1 - 0.5\tanh(x)) )_x + 0.05(u_x/u)_x$): A macroscopic traffic model incorporating a physical road narrowing ($\tanh$) and driver anticipation gradients.
    
    \item \textbf{Cross-Chemotactic Predator-Prey System} ($u_t = 0.05 u_{xx} + u(1 - u - 0.5v)$ and $v_t = 0.01 v_{xx} - 0.2(v u_x)_x + 0.8v(u - 1)$): A coupled multivariate system where predators ($v$) actively hunt by moving along the prey's ($u$) density gradient.
\end{itemize}

\begin{table}[htbp]
    \centering
    \caption{\textbf{Evaluation on structurally novel equations.} Assessed via Recovery Status (\textbf{Stat.}: $\color{ForestGreen}{\bigstar}$ Exact, $\color{orange}{\mathbf{\approx}}$ Partial, $\color{red}{\boldsymbol{\times}}$ Fail), \textbf{MSE}, and complexity (\textbf{C}).}
    \label{tab:16col_evaluation}
    \renewcommand{\arraystretch}{1.3}
    \setlength{\tabcolsep}{2pt}
    \small
    \begin{tabular}{l ccc ccc ccc ccc ccc}
        \toprule
        
        \multirow{2}{*}{\textbf{Algorithm}} 
        & \multicolumn{3}{c}{\begin{tabular}{@{}c@{}}\textbf{Topography} \\ \textbf{Chemotaxis}\end{tabular}} 
        & \multicolumn{3}{c}{\begin{tabular}{@{}c@{}}\textbf{Morpho-} \\ \textbf{genesis}\end{tabular}} 
        & \multicolumn{3}{c}{\begin{tabular}{@{}c@{}}\textbf{Forced Swift-} \\ \textbf{Hohenberg}\end{tabular}} 
        & \multicolumn{3}{c}{\begin{tabular}{@{}c@{}}\textbf{Traffic Flow} \\ \textbf{Bottleneck}\end{tabular}} 
        & \multicolumn{3}{c}{\begin{tabular}{@{}c@{}}\textbf{Predator-Prey} \\ \textbf{(Coupled)}\end{tabular}} \\
        
        \cmidrule(lr){2-4} \cmidrule(lr){5-7} \cmidrule(lr){8-10} \cmidrule(lr){11-13} \cmidrule(lr){14-16}
        
        & Stat. & MSE & $C$ & Stat. & MSE & $C$ & Stat. & MSE & $C$ & Stat. & MSE & $C$ & Stat. & MSE & $C$ \\
        \midrule
        
        SGA-PDE
        & $\color{red}{\boldsymbol{\times}}$ & 4.3e-4 & 20 
        & $\color{red}{\boldsymbol{\times}}$ & 4.4e-3 & 19 
        & $\color{red}{\boldsymbol{\times}}$ & 1.74e-3 & \textbf{1} 
        & $\color{red}{\boldsymbol{\times}}$ & 7.1e-4 & 14 
        & $\color{red}{\boldsymbol{\times}}$ & N/A & -- \\
        
        DISCOVER 
        & $\color{red}{\boldsymbol{\times}}$ & 4.2e-4 & 15 
        & $\color{red}{\boldsymbol{\times}}$ & 5.6e-4 & 17 
        & $\color{red}{\boldsymbol{\times}}$ & 4.4e-4 & 28 
        & $\color{red}{\boldsymbol{\times}}$ & 3.1e-4 & 29 
        & $\color{red}{\boldsymbol{\times}}$ & N/A &  -- \\
        
        WSINDy\_PDE 
        & $\color{red}{\boldsymbol{\times}}$ & 1.1e-1 & 106 
        & $\color{red}{\boldsymbol{\times}}$ & 5.4e-3 & 234 
        & $\color{red}{\boldsymbol{\times}}$ & 1.8e-3 & 234 
        & $\color{red}{\boldsymbol{\times}}$ & 6.9e-2 & 234 
        & $\color{red}{\boldsymbol{\times}}$ & 1.7e-3 & 203 \\
        
        EqGPT 
        & $\color{red}{\boldsymbol{\times}}$ & 1.6e-3 & \textbf{10} 
        & $\color{red}{\boldsymbol{\times}}$ & 5.1e-3 & \textbf{16} 
        & $\color{red}{\boldsymbol{\times}}$ & 5.1e-3 & 7
        & $\color{red}{\boldsymbol{\times}}$ & 6.3e-4 & \textbf{11} 
        & $\color{red}{\boldsymbol{\times}}$ & N/A & --  \\
        
        \midrule
        
        \textbf{LLM-PDESR} 
        & $\color{orange}{\mathbf{\approx}}$ & \textbf{7.6e-5} & 16 
        & $\color{ForestGreen}{\bigstar}$ & \textbf{1.2e-10} & 21 
        & $\color{orange}{\mathbf{\approx}}$ & \textbf{2.2e-4} & 26 
        & $\color{red}{\boldsymbol{\times}}$ & \textbf{1.6e-4} & 23 
        & $\color{ForestGreen}{\bigstar}$ & \textbf{4.4e-9} & \textbf{47}\\
        
        \bottomrule
    \end{tabular}%
    
\end{table}

As Table~\ref{tab:16col_evaluation} reports, the evaluated baselines universally fail to recover these structurally novel equations. Specifically, WSINDy\_PDE succumbs to combinatorial explosion and extreme overfitting (e.g., yielding an AST complexity of 203 with high MSE for Predator-Prey). Meanwhile, SGA-PDE, DISCOVER, and EqGPT suffer from structural collapse and fundamentally fail to discover the correct forms. Furthermore, as most baselines lack the architectural support for multivariate systems, they are inherently inapplicable to the coupled Predator-Prey dynamics (denoted as N/A).


Conversely, LLM-PDESR successfully navigates this vast hypothesis space. By accommodating inherent structural complexities, our framework extracts the exact mathematical skeletons ($\color{ForestGreen}{\bigstar}$) for both Morphogenesis and the coupled Predator-Prey dynamics without generating spurious terms. Notably, even when exact symbolic recovery is unachieved for highly nonlinear systems like the Traffic Flow Bottleneck ($\color{red}{\boldsymbol{\times}}$), LLM-PDESR consistently secures the lowest MSE across all five systems, demonstrating superior capability in approximating underlying physical dynamics.

\textit{Remark on Partial Recovery:} For Topography Chemotaxis and Forced Swift-Hohenberg, LLM-PDESR achieves partial recovery ($\color{orange}{\mathbf{\approx}}$), strictly defined as correctly identifying $\ge 50\%$ of the exact governing terms. Rather than producing unphysical terms when faced with complex interactions, our framework exhibits exceptional structural robustness via Pareto-driven optimization. It reliably isolates core physical mechanisms and prevents mathematically invalid bloat, even when the complete skeleton proves exceptionally intricate. A comprehensive, term-by-term analysis and further detailed discussions regarding the recovered structures for all five novel systems are provided in Appendix \ref{Appendix:E}.

\subsection{Real-world discovery: coupled atmospheric dynamics}

Evaluating symbolic regression on authentic reanalysis data like ERA5 is notoriously challenging, as the algorithm must contend with severe dimensionality truncation (unresolved 3D thermodynamic variables) and high-frequency data assimilation noise. To demonstrate our framework's capability in this extreme regime, we deployed LLM-PDESR on real-world ERA5 atmospheric data (North Pacific domain; Jan 2023) \cite{munoz2021era5,lavers2022evaluation,soci2024era5}, autonomously extracting an interpretable 1D coupled atmospheric circulation surrogate (1D-CACE):

\begin{equation}
\label{eq:cace}
\begin{aligned}
    u_t &= \theta_1 u_{xx} + \theta_2 u u_x + \theta_3 v u_x + \theta_4 v_x + \theta_5 u + \theta_6 v + \theta_7 x v \\
    v_t &= \theta_8 v_{xx} + \theta_9 v v_x + \theta_{10} u v_x + \theta_{11} u_x + \theta_{12} v + \theta_{13} u + \theta_{14} x u
\end{aligned}
\end{equation}

where $x$ is the spatial coordinate and $\theta_{1}\dots\theta_{14}$ are optimizable physical parameters. Although applicable to multivariate systems, WSINDy\_PDE yields a severely overfitted, unphysical expression here, omitted for brevity.

Unlike standard 1D models (e.g., Burgers equation) limited to self-advection, 1D-CACE explicitly captures cross-wind momentum transfer ($v u_x,  u v_x$) and linear coupling ($ v_x,  u_x$). From a Reduced Order Modeling (ROM) perspective, projecting complex 3D atmospheric flows onto a 1D manifold inevitably loses external physical forcings. Remarkably, LLM-PDESR autonomously compensates for this by discovering spatially dependent terms ($x v$ and $x u$). Rather than just overfitting localized noise, these act as vital mathematical closure terms, implicitly proxying unresolved macroscopic background gradients (analogous to the well-known $\beta$-plane approximation of the Coriolis force). Detailed physical derivations of these compensation terms are provided in Appendix \ref{Appendix:F}.

This dynamical consistency ensures strong out-of-distribution (OOD) generalization. By freezing the structural skeleton of Eq.~\eqref{eq:cace} and re-optimizing only the numeric coefficients, the equation transfers effectively to distinct spatiotemporal domains (Table~\ref{tab:ood_generalization}).

\begin{table}[ht]
  \caption{\textbf{OOD generalization of the 1D-CACE skeleton.} The structural skeleton is frozen while numeric parameters are re-optimized for each regime to accommodate altered background states.}
  \label{tab:ood_generalization}
  \centering
  \renewcommand{\arraystretch}{1.3}
  \begin{tabular}{llllc}
    \toprule
    \textbf{Evaluation Regime} & \textbf{Time} & \textbf{Region} & \textbf{Type} & \textbf{$R^2$} \\
    \midrule
    North Pacific (Base) & Jan 1--10 & 25--45$^\circ$N, 160--130$^\circ$W & In-Distribution & \textbf{76.28\%} \\
    West Pacific         & Jan 1--10 & 25--45$^\circ$N, 140--170$^\circ$E & Spatial OOD     & 60.97\% \\
    North Pacific        & Feb 1--10 & 25--45$^\circ$N, 160--130$^\circ$W & Temporal Drift  & 58.64\% \\
    South Indian Ocean   & Jul 1--10 & 40--60$^\circ$S, 60--90$^\circ$E   & \textbf{Hemispheric} & \textbf{72.54\%} \\
    \bottomrule
  \end{tabular}
\end{table}

\begin{figure}[htbp]
    \centering
    \includegraphics[width=\textwidth]{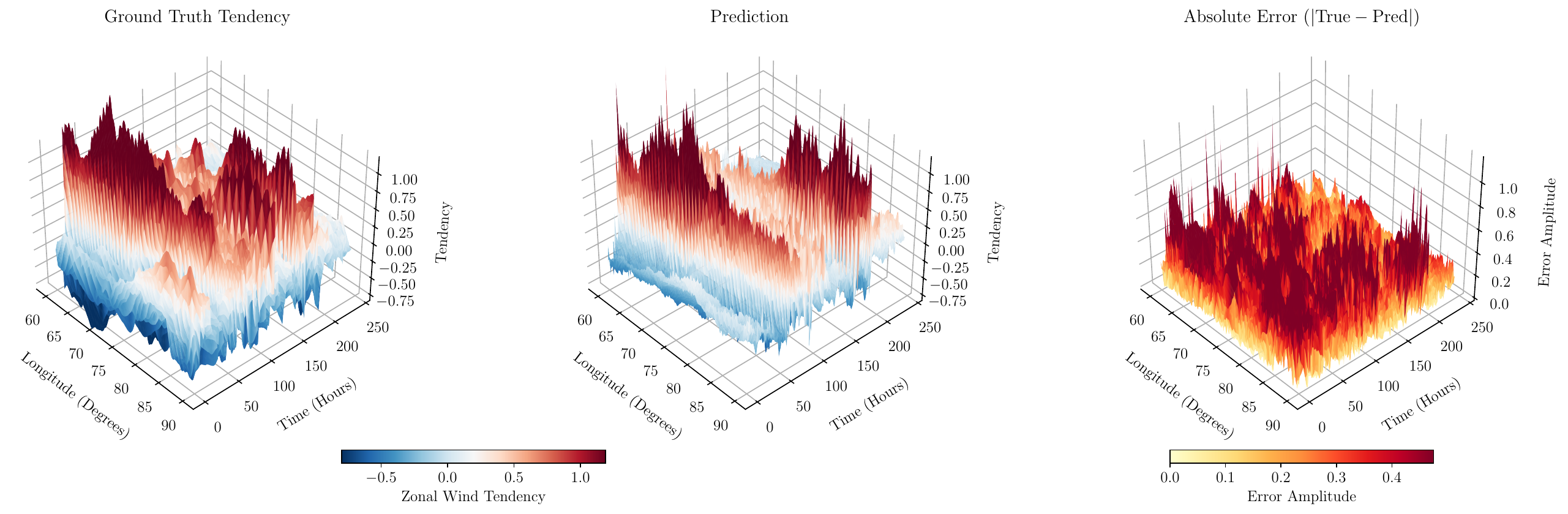} 
    \caption{\textbf{Hemispheric OOD predictions of 1D-CACE.} The model robustly reconstructs South Indian Ocean dynamics (austral winter) despite severe geographical shifts.}
    \label{fig:era5_heatmaps}
\end{figure}

Notably, while temporal drift to February causes a performance drop ($R^2=58.64\%$) due to transitional late-winter dynamics, the geographically extreme Hemispheric OOD setting (South Indian Ocean, July) rebounds significantly to $72.54\%$. Because July represents the austral winter, this recovery empirically confirms that 1D-CACE successfully captures the macroscopic dynamics of \textit{winter atmospheric circulation}, rather than memorizing geographical artifacts. Figure~\ref{fig:era5_heatmaps} visualizes this high-fidelity hemispheric reconstruction, verifying that LLM-PDESR effectively distills interpretable governing dynamics from chaotic real-world data.

\subsection{Ablation studies}
\label{sec:ablation}

\begin{figure}[htbp]
    \centering
    \includegraphics[width=\textwidth]{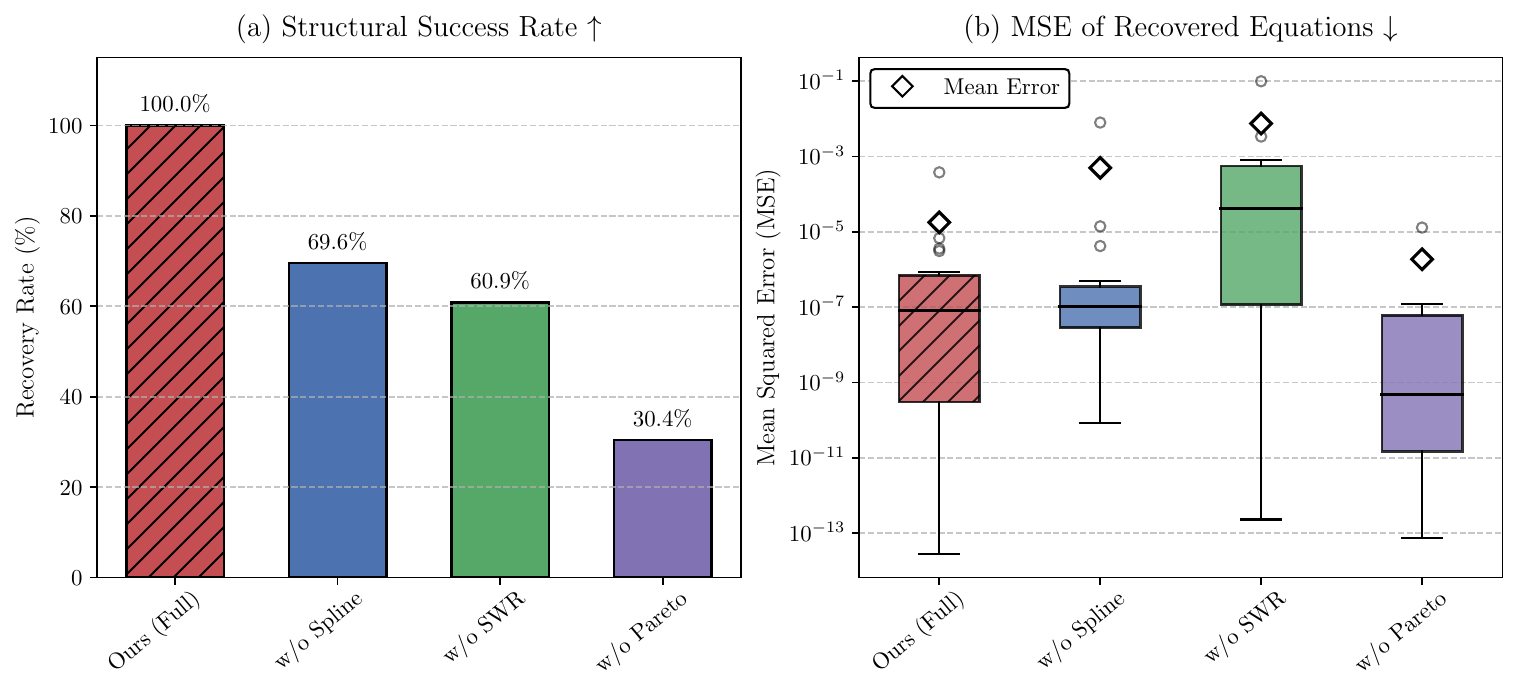} 
    \caption{\textbf{Ablation study of the LLM-PDESR framework.} Impact of removing core components on (a) structural success rate and (b) MSE distributions under noise-free conditions.}
    \label{fig:ablation}
\end{figure}

To systematically evaluate the individual contribution of each core component, we conduct an ablation study using noise-free data. We compare the full LLM-PDESR framework against three ablated variants: (1) w/o Spline (replacing quintic splines with standard FDs), (2) w/o SWR (replacing local integral residuals with pointwise MSE), and (3) w/o Pareto (removing the multi-objective selection, retaining only MSE as the fitness metric).

Figure~\ref{fig:ablation}(a) illustrates the significant impact of removing any single module on the \textit{Structural Success Rate}. While the full framework achieves a $100\%$ recovery on the benchmark PDEs, reverting to discrete FDs (w/o Spline) or pointwise evaluation (w/o SWR) significantly degrades success rates to $69.6\%$ and $60.9\%$, respectively. 

Figure~\ref{fig:ablation}(b) elucidates the underlying optimization dynamics through the MSE distributions. Without splines, numerical differentiation errors rapidly accumulate. Similarly, without the SWR's inherent smoothing, the fitness landscape is severely distorted by high-frequency numerical artifacts and truncation errors (even in the absence of observation noise), resulting in high MSE variances and optimization failures. Most crucially, the w/o Pareto variant exhibits a deceptive phenomenon: while its median MSE appears competitively low, its structural recovery remains poor. This corroborates our theoretical analysis: without a strict complexity penalty, the LLM inevitably falls into an overfitting trap. It produces bloated, mathematically spurious equations that artificially minimize the residual error, completely failing to execute genuine symbolic discovery.
\section{Conclusion and future work}
\label{sec:conclusion}

In this work, we introduce LLM-PDESR, a novel paradigm bridging the structural hypothesis generation of LLMs with the numerical rigor required for robust PDE discovery. By integrating $\mathcal{C}^4$ continuous quintic splines, SWR evaluations, and a Pareto-driven automated selection loop, our framework effectively mitigates the fitness landscape distortion and the generation of spurious equations that plague traditional methods. Empirically, it achieves state-of-the-art performance across canonical benchmarks and structurally novel systems (explicitly designed to mitigate the risk of training data contamination), while demonstrating practical utility by extracting an interpretable 1D coupled atmospheric circulation surrogate (1D-CACE) directly from noisy ERA5 observations. While our current formulation successfully distills macroscopic dynamics from noisy real-world environments, future work will focus on extending this framework to high-dimensional (2D/3D) physical systems on unstructured grids and optimizing the computational overhead of iterative LLM inference, thereby further advancing the frontier of automated scientific discovery.

\section{Acknowledgments}
This work was supported by the National Science and Technology
Major Project under Grant 2021ZD0112500. We thank the Computing
Center of Jilin Province for technical support.

{
\small 
\bibliographystyle{unsrtnat} 
\bibliography{refs} 

\begin{thebibliography}{50}
\providecommand{\natexlab}[1]{#1}
\providecommand{\url}[1]{\texttt{#1}}
\expandafter\ifx\csname urlstyle\endcsname\relax
  \providecommand{\doi}[1]{doi: #1}\else
  \providecommand{\doi}{doi: \begingroup \urlstyle{rm}\Url}\fi

\bibitem[Rudy et~al.(2017)Rudy, Brunton, Proctor, and Kutz]{rudy2017data}
Samuel~H Rudy, Steven~L Brunton, Joshua~L Proctor, and J~Nathan Kutz.
\newblock Data-driven discovery of partial differential equations.
\newblock \emph{Science advances}, 3\penalty0 (4):\penalty0 e1602614, 2017.

\bibitem[Long et~al.(2018)Long, Lu, Ma, and Dong]{long2018pde}
Zichao Long, Yiping Lu, Xianzhong Ma, and Bin Dong.
\newblock Pde-net: Learning pdes from data.
\newblock In \emph{International conference on machine learning}, pages
  3208--3216. PMLR, 2018.

\bibitem[Raissi et~al.(2019)Raissi, Perdikaris, and
  Karniadakis]{raissi2019physics}
Maziar Raissi, Paris Perdikaris, and George~E Karniadakis.
\newblock Physics-informed neural networks: A deep learning framework for
  solving forward and inverse problems involving nonlinear partial differential
  equations.
\newblock \emph{Journal of Computational physics}, 378:\penalty0 686--707,
  2019.

\bibitem[Both et~al.(2021)Both, Choudhury, Sens, and Kusters]{both2021deepmod}
Gert-Jan Both, Subham Choudhury, Pierre Sens, and Remy Kusters.
\newblock Deepmod: Deep learning for model discovery in noisy data.
\newblock \emph{Journal of Computational Physics}, 428:\penalty0 109985, 2021.

\bibitem[Karniadakis et~al.(2021)Karniadakis, Kevrekidis, Lu, Perdikaris, Wang,
  and Yang]{karniadakis2021physics}
George~Em Karniadakis, Ioannis~G Kevrekidis, Lu~Lu, Paris Perdikaris, Sifan
  Wang, and Liu Yang.
\newblock Physics-informed machine learning.
\newblock \emph{Nature Reviews Physics}, 3\penalty0 (6):\penalty0 422--440,
  2021.

\bibitem[Dong and Zhong(2025)]{dong2025recent}
Junlan Dong and Jinghui Zhong.
\newblock Recent advances in symbolic regression.
\newblock \emph{ACM Computing Surveys}, 57\penalty0 (11):\penalty0 1--37, 2025.

\bibitem[Shojaee et~al.(2023)Shojaee, Meidani, Barati~Farimani, and
  Reddy]{shojaee2023transformer}
Parshin Shojaee, Kazem Meidani, Amir Barati~Farimani, and Chandan Reddy.
\newblock Transformer-based planning for symbolic regression.
\newblock \emph{Advances in Neural Information Processing Systems},
  36:\penalty0 45907--45919, 2023.

\bibitem[Xiang et~al.(2025)Xiang, Ashen, Qian, and Qian]{xiang2025graph}
Ziyu Xiang, Kenna Ashen, Xiaofeng Qian, and Xiaoning Qian.
\newblock Graph-based symbolic regression with invariance and constraint
  encoding.
\newblock In \emph{The Thirty-ninth Annual Conference on Neural Information
  Processing Systems}, 2025.

\bibitem[Brunton et~al.(2016)Brunton, Proctor, and
  Kutz]{brunton2016discovering}
Steven~L Brunton, Joshua~L Proctor, and J~Nathan Kutz.
\newblock Discovering governing equations from data by sparse identification of
  nonlinear dynamical systems.
\newblock \emph{Proceedings of the national academy of sciences}, 113\penalty0
  (15):\penalty0 3932--3937, 2016.

\bibitem[Bottmer et~al.(2022)Bottmer, Croux, and Wilms]{bottmer2022sparse}
Lea Bottmer, Christophe Croux, and Ines Wilms.
\newblock Sparse regression for large data sets with outliers.
\newblock \emph{European Journal of Operational Research}, 297\penalty0
  (2):\penalty0 782--794, 2022.

\bibitem[Zhong et~al.(2018)Zhong, Feng, Cai, and Ong]{zhong2018multifactorial}
Jinghui Zhong, Liang Feng, Wentong Cai, and Yew-Soon Ong.
\newblock Multifactorial genetic programming for symbolic regression problems.
\newblock \emph{IEEE transactions on systems, man, and cybernetics: systems},
  50\penalty0 (11):\penalty0 4492--4505, 2018.

\bibitem[Du et~al.(2025)Du, Liu, Cheng, Li, and Yu]{du2025deep}
Jinyang Du, Renyun Liu, Du~Cheng, Qingliang Li, and Fanhua Yu.
\newblock Deep recognition of partial differential equations based on
  reinforcement learning and genetic algorithm.
\newblock \emph{The Journal of Supercomputing}, 81\penalty0 (5):\penalty0 740,
  2025.

\bibitem[Mei et~al.(2022)Mei, Chen, Lensen, Xue, and Zhang]{mei2022explainable}
Yi~Mei, Qi~Chen, Andrew Lensen, Bing Xue, and Mengjie Zhang.
\newblock Explainable artificial intelligence by genetic programming: A survey.
\newblock \emph{IEEE Transactions on Evolutionary Computation}, 27\penalty0
  (3):\penalty0 621--641, 2022.

\bibitem[Makke and Chawla(2024)]{makke2024interpretable}
Nour Makke and Sanjay Chawla.
\newblock Interpretable scientific discovery with symbolic regression: a
  review.
\newblock \emph{Artificial Intelligence Review}, 57\penalty0 (1):\penalty0 2,
  2024.

\bibitem[Romera-Paredes et~al.(2024)Romera-Paredes, Barekatain, Novikov, Balog,
  Kumar, Dupont, Ruiz, Ellenberg, Wang, Fawzi, et~al.]{romera2024mathematical}
Bernardino Romera-Paredes, Mohammadamin Barekatain, Alexander Novikov, Matej
  Balog, M~Pawan Kumar, Emilien Dupont, Francisco~JR Ruiz, Jordan~S Ellenberg,
  Pengming Wang, Omar Fawzi, et~al.
\newblock Mathematical discoveries from program search with large language
  models.
\newblock \emph{Nature}, 625\penalty0 (7995):\penalty0 468--475, 2024.

\bibitem[Liu et~al.(2024)Liu, Yuan, Hu, Li, Chen, Liu, Lu, and Jia]{liu2024rl}
Shaoteng Liu, Haoqi Yuan, Minda Hu, Yanwei Li, Yukang Chen, Shu Liu, Zongqing
  Lu, and Jiaya Jia.
\newblock Rl-gpt: Integrating reinforcement learning and code-as-policy.
\newblock \emph{Advances in Neural Information Processing Systems},
  37:\penalty0 28430--28459, 2024.

\bibitem[Wang et~al.(2023)Wang, Fu, Du, Gao, Huang, Liu, Chandak, Liu,
  Van~Katwyk, Deac, et~al.]{wang2023scientific}
Hanchen Wang, Tianfan Fu, Yuanqi Du, Wenhao Gao, Kexin Huang, Ziming Liu, Payal
  Chandak, Shengchao Liu, Peter Van~Katwyk, Andreea Deac, et~al.
\newblock Scientific discovery in the age of artificial intelligence.
\newblock \emph{Nature}, 620\penalty0 (7972):\penalty0 47--60, 2023.

\bibitem[AI4Science and Quantum(2023)]{ai4science2023impact}
Microsoft~Research AI4Science and Microsoft~Azure Quantum.
\newblock The impact of large language models on scientific discovery: a
  preliminary study using gpt-4.
\newblock \emph{arXiv preprint arXiv:2311.07361}, 2023.

\bibitem[Wang et~al.(2024)Wang, Downey, Ji, and Hope]{wang2024scimon}
Qingyun Wang, Doug Downey, Heng Ji, and Tom Hope.
\newblock Scimon: Scientific inspiration machines optimized for novelty.
\newblock In \emph{Proceedings of the 62nd Annual Meeting of the Association
  for Computational Linguistics (Volume 1: Long Papers)}, pages 279--299, 2024.

\bibitem[Shojaee et~al.(2024)Shojaee, Meidani, Gupta, Farimani, and
  Reddy]{shojaee2024llm}
Parshin Shojaee, Kazem Meidani, Shashank Gupta, Amir~Barati Farimani, and
  Chandan~K Reddy.
\newblock Llm-sr: Scientific equation discovery via programming with large
  language models.
\newblock \emph{arXiv preprint arXiv:2404.18400}, 2024.

\bibitem[Teixeira et~al.(2023)Teixeira, Sarris, Zhang, Na, Berenger, Su,
  Okoniewski, Chew, Backman, and Simpson]{teixeira2023finite}
FL~Teixeira, C~Sarris, Y~Zhang, D-Y Na, J-P Berenger, Y~Su, M~Okoniewski,
  WC~Chew, V~Backman, and Jamesina~J Simpson.
\newblock Finite-difference time-domain methods.
\newblock \emph{Nature Reviews Methods Primers}, 3\penalty0 (1):\penalty0 75,
  2023.

\bibitem[LeVeque(1998)]{leveque1998finite}
Randall~J LeVeque.
\newblock Finite difference methods for differential equations.
\newblock \emph{Draft version for use in AMath}, 585\penalty0 (6):\penalty0
  112, 1998.

\bibitem[Patidar(2005)]{patidar2005use}
Kailash~C Patidar.
\newblock On the use of nonstandard finite difference methods.
\newblock \emph{Journal of Difference Equations and Applications}, 11\penalty0
  (8):\penalty0 735--758, 2005.

\bibitem[Othmane et~al.(2022)Othmane, Kiltz, and Rudolph]{othmane2022survey}
Amine Othmane, Lothar Kiltz, and Joachim Rudolph.
\newblock Survey on algebraic numerical differentiation: historical
  developments, parametrization, examples, and applications.
\newblock \emph{International Journal of Systems Science}, 53\penalty0
  (9):\penalty0 1848--1887, 2022.

\bibitem[Kaheman et~al.(2022)Kaheman, Brunton, and
  Nathan~Kutz]{kaheman2022automatic}
Kadierdan Kaheman, Steven~L Brunton, and J~Nathan~Kutz.
\newblock Automatic differentiation to simultaneously identify nonlinear
  dynamics and extract noise probability distributions from data.
\newblock \emph{Machine Learning: Science and Technology}, 3\penalty0
  (1):\penalty0 015031, 2022.

\bibitem[Alam et~al.(2025)Alam, Khan, and Roul]{alam2025high}
Mohammad~Prawesh Alam, Arshad Khan, and Pradip Roul.
\newblock High-resolution numerical method for the time-fractional fourth-order
  diffusion problems via improved quintic b-spline function.
\newblock \emph{Journal of Applied Mathematics and Computing}, 71\penalty0
  (1):\penalty0 133--171, 2025.

\bibitem[Erkorkmaz and Altintas(2005)]{erkorkmaz2005quintic}
Kaan Erkorkmaz and Yusuf Altintas.
\newblock Quintic spline interpolation with minimal feed fluctuation.
\newblock \emph{J. Manuf. Sci. Eng.}, 127\penalty0 (2):\penalty0 339--349,
  2005.

\bibitem[Messenger and Bortz(2025)]{messenger2025asymptotic}
Daniel~A Messenger and David~M Bortz.
\newblock Asymptotic consistency of the wsindy algorithm in the limit of
  continuum data.
\newblock \emph{IMA Journal of Numerical Analysis}, 45\penalty0 (6):\penalty0
  3264--3312, 2025.

\bibitem[Xu et~al.(2023)Xu, Zhang, Huang, Qu, and Nojima]{xu2023pareto}
Ying Xu, Huan Zhang, Lei Huang, Rong Qu, and Yusuke Nojima.
\newblock A pareto front grid guided multi-objective evolutionary algorithm.
\newblock \emph{Applied Soft Computing}, 136:\penalty0 110095, 2023.

\bibitem[Liu et~al.(2025)Liu, Wu, Huang, Gao, Wang, Xue, and
  Qian]{liu2025pareto}
Erlong Liu, Yu-Chang Wu, Xiaobin Huang, Chengrui Gao, Ren-Jian Wang, Ke~Xue,
  and Chao Qian.
\newblock Pareto set learning for multi-objective reinforcement learning.
\newblock In \emph{Proceedings of the AAAI Conference on Artificial
  Intelligence}, volume~39, pages 18789--18797, 2025.

\bibitem[Hersbach et~al.(2020)Hersbach, Bell, Berrisford, Hirahara,
  Hor{\'a}nyi, Mu{\~n}oz-Sabater, Nicolas, Peubey, Radu, Schepers,
  et~al.]{hersbach2020era5}
Hans Hersbach, Bill Bell, Paul Berrisford, Shoji Hirahara, Andr{\'a}s
  Hor{\'a}nyi, Joaqu{\'\i}n Mu{\~n}oz-Sabater, Julien Nicolas, Carole Peubey,
  Raluca Radu, Dinand Schepers, et~al.
\newblock The era5 global reanalysis.
\newblock \emph{Quarterly journal of the royal meteorological society},
  146\penalty0 (730):\penalty0 1999--2049, 2020.

\bibitem[Chen et~al.(2022)Chen, Luo, Liu, Xu, and Zhang]{chen2022symbolic}
Yuntian Chen, Yingtao Luo, Qiang Liu, Hao Xu, and Dongxiao Zhang.
\newblock Symbolic genetic algorithm for discovering open-form partial
  differential equations (sga-pde).
\newblock \emph{Physical Review Research}, 4\penalty0 (2):\penalty0 023174,
  2022.

\bibitem[Du et~al.(2024)Du, Chen, and Zhang]{du2024discover}
Mengge Du, Yuntian Chen, and Dongxiao Zhang.
\newblock Discover: Deep identification of symbolically concise open-form
  partial differential equations via enhanced reinforcement learning.
\newblock \emph{Physical Review Research}, 6\penalty0 (1):\penalty0 013182,
  2024.

\bibitem[Messenger and Bortz(2021)]{messenger2021weak}
Daniel~A Messenger and David~M Bortz.
\newblock Weak sindy for partial differential equations.
\newblock \emph{Journal of Computational Physics}, 443:\penalty0 110525, 2021.

\bibitem[Xu et~al.(2025)Xu, Chen, Cao, Tang, Du, Li, Callaghan, and
  Zhang]{xu2025generative}
Hao Xu, Yuntian Chen, Rui Cao, Tianning Tang, Mengge Du, Jian Li, Adrian~H
  Callaghan, and Dongxiao Zhang.
\newblock Generative discovery of partial differential equations by learning
  from math handbooks.
\newblock \emph{Nature Communications}, 16\penalty0 (1):\penalty0 10255, 2025.

\bibitem[Li et~al.(2020)Li, Kovachki, Azizzadenesheli, Liu, Bhattacharya,
  Stuart, and Anandkumar]{li2020fourier}
Zongyi Li, Nikola Kovachki, Kamyar Azizzadenesheli, Burigede Liu, Kaushik
  Bhattacharya, Andrew Stuart, and Anima Anandkumar.
\newblock Fourier neural operator for parametric partial differential
  equations.
\newblock \emph{arXiv preprint arXiv:2010.08895}, 2020.

\bibitem[Lu et~al.(2021)Lu, Jin, Pang, Zhang, and Karniadakis]{lu2021learning}
Lu~Lu, Pengzhan Jin, Guofei Pang, Zhongqiang Zhang, and George~Em Karniadakis.
\newblock Learning nonlinear operators via deeponet based on the universal
  approximation theorem of operators.
\newblock \emph{Nature machine intelligence}, 3\penalty0 (3):\penalty0
  218--229, 2021.

\bibitem[Schmidt and Lipson(2009)]{schmidt2009distilling}
Michael Schmidt and Hod Lipson.
\newblock Distilling free-form natural laws from experimental data.
\newblock \emph{science}, 324\penalty0 (5923):\penalty0 81--85, 2009.

\bibitem[Cranmer(2023)]{cranmer2023interpretable}
Miles Cranmer.
\newblock Interpretable machine learning for science with pysr and
  symbolicregression. jl.
\newblock \emph{arXiv preprint arXiv:2305.01582}, 2023.

\bibitem[Maslyaev et~al.(2021)Maslyaev, Hvatov, and
  Kalyuzhnaya]{maslyaev2021partial}
Mikhail Maslyaev, Alexander Hvatov, and Anna~V Kalyuzhnaya.
\newblock Partial differential equations discovery with epde framework:
  Application for real and synthetic data.
\newblock \emph{Journal of Computational Science}, 53:\penalty0 101345, 2021.

\bibitem[Petersen et~al.(2019)Petersen, Landajuela, Mundhenk, Santiago, Kim,
  and Kim]{petersen2019deep}
Brenden~K Petersen, Mikel Landajuela, T~Nathan Mundhenk, Claudio~P Santiago,
  Soo~K Kim, and Joanne~T Kim.
\newblock Deep symbolic regression: Recovering mathematical expressions from
  data via risk-seeking policy gradients.
\newblock \emph{arXiv preprint arXiv:1912.04871}, 2019.

\bibitem[Zhang et~al.(2024)Zhang, Guan, Liu, and Wang]{zhang2024morl4pdes}
Xiaoxia Zhang, Junsheng Guan, Yanjun Liu, and Guoyin Wang.
\newblock Morl4pdes: Data-driven discovery of pdes based on multi-objective
  optimization and reinforcement learning.
\newblock \emph{Chaos, Solitons \& Fractals}, 180:\penalty0 114536, 2024.

\bibitem[Chicco et~al.(2021)Chicco, Warrens, and Jurman]{chicco2021coefficient}
Davide Chicco, Matthijs~J Warrens, and Giuseppe Jurman.
\newblock The coefficient of determination r-squared is more informative than
  smape, mae, mape, mse and rmse in regression analysis evaluation.
\newblock \emph{Peerj computer science}, 7:\penalty0 e623, 2021.

\bibitem[Qiao et~al.(2025)Qiao, Yan, and Huang]{qiao2025hybrid}
Xin Qiao, Qingyun Yan, and Weimin Huang.
\newblock Hybrid cnn-transformer network with a weighted mse loss for global
  sea surface wind speed retrieval from gnss-r data.
\newblock \emph{IEEE Transactions on Geoscience and Remote Sensing}, 2025.

\bibitem[LeVeque(2007)]{leveque2007finite}
Randall~J LeVeque.
\newblock \emph{Finite difference methods for ordinary and partial differential
  equations: steady-state and time-dependent problems}.
\newblock SIAM, 2007.

\bibitem[Nocedal and Wright(2006)]{nocedal2006numerical}
Jorge Nocedal and Stephen~J Wright.
\newblock \emph{Numerical optimization}.
\newblock Springer, 2006.

\bibitem[Mannel et~al.(2024)Mannel, Om~Aggrawal, and
  Modersitzki]{mannel2024structured}
Florian Mannel, Hari Om~Aggrawal, and Jan Modersitzki.
\newblock A structured l-bfgs method and its application to inverse problems.
\newblock \emph{Inverse Problems}, 40\penalty0 (4):\penalty0 045022, 2024.

\bibitem[Mu{\~n}oz-Sabater et~al.(2021)Mu{\~n}oz-Sabater, Dutra,
  Agust{\'\i}-Panareda, Albergel, Arduini, Balsamo, Boussetta, Choulga,
  Harrigan, Hersbach, et~al.]{munoz2021era5}
Joaqu{\'\i}n Mu{\~n}oz-Sabater, Emanuel Dutra, Anna Agust{\'\i}-Panareda,
  Cl{\'e}ment Albergel, Gabriele Arduini, Gianpaolo Balsamo, Souhail Boussetta,
  Margarita Choulga, Shaun Harrigan, Hans Hersbach, et~al.
\newblock Era5-land: A state-of-the-art global reanalysis dataset for land
  applications.
\newblock \emph{Earth system science data}, 13\penalty0 (9):\penalty0
  4349--4383, 2021.

\bibitem[Lavers et~al.(2022)Lavers, Simmons, Vamborg, and
  Rodwell]{lavers2022evaluation}
David~A Lavers, Adrian Simmons, Freja Vamborg, and Mark~J Rodwell.
\newblock An evaluation of era5 precipitation for climate monitoring.
\newblock \emph{Quarterly Journal of the Royal Meteorological Society},
  148\penalty0 (748):\penalty0 3152--3165, 2022.

\bibitem[Soci et~al.(2024)Soci, Hersbach, Simmons, Poli, Bell, Berrisford,
  Hor{\'a}nyi, Mu{\~n}oz-Sabater, Nicolas, Radu, et~al.]{soci2024era5}
Cornel Soci, Hans Hersbach, Adrian Simmons, Paul Poli, Bill Bell, Paul
  Berrisford, Andr{\'a}s Hor{\'a}nyi, Joaqu{\'\i}n Mu{\~n}oz-Sabater, Julien
  Nicolas, Raluca Radu, et~al.
\newblock The era5 global reanalysis from 1940 to 2022.
\newblock \emph{Quarterly Journal of the Royal Meteorological Society},
  150\penalty0 (764):\penalty0 4014--4048, 2024.

\end{thebibliography}
}


\appendix

\section{Prompt design and numerical evaluation pipeline}
\label{Appendix:A}

This section details the operating mechanism of LLM-PDESR. The process is decoupled into two core components: the structured prompt template used to guide the LLM in generating candidate equations, and the rigorous numerical pipeline used to evaluate and optimize the parameters of these candidates.

\subsection{Structured LLM prompt for equation generation}

To concretize our methodology, we illustrate the prompt structure using a representative and highly challenging governing equation from our experiments: \textbf{Topography-Constrained Chemotaxis}. The target PDE is given by:
\begin{equation}
    u_t = 0.1 u_{xx} - 0.5 ( u(1-u) \cos(x) )_x + u(1-u)
\end{equation}
This equation models biological foraging on an undulating terrain coupled with logistic crowding. Crucially, its nested derivative structure—specifically the advection term $(u(1-u)\cos(x))_x$—inherently prevents simple additive token generation, thereby severely testing the LLM's capacity for deep physical reasoning and symbolic manipulation.

To prevent the generation of physically meaningless or computationally intractable expressions, we constrain the LLM's action space through a highly structured prompt. The prompt enforces strict mathematical vocabularies and introduces the Pareto evolution mechanism, which organically penalizes bloated equations. The exact prompt template tailored for this specific biological task is provided in \textbf{Prompt \ref{prompt:pde_discovery}}.

To rigorously evaluate the generalizability of our framework and ensure reproducibility, we apply this standardized prompt template across all structurally novel systems. Crucially, the Strict Constraints and Pareto Evolution Mechanism remain completely identical across all trials to prevent task-specific prompt engineering bias. The LLM's initial context is altered solely by modifying the physical background provided in the Task description.

\begin{tcolorbox}[colback=gray!5!white,colframe=black!75!black,title=\textbf{Prompt 1:} Structured Template for PDE Discovery (Example: Topography-Constrained Chemotaxis), label=prompt:pde_discovery]

\textbf{System Role:} You are a physicist discovering PDEs.

\textbf{Task:} Find the mathematical function skeleton that represents the time derivative of the field variable ($u_t$), given data on spatial coordinates ($x$) and the field variable ($u$) in a biological population migrating along a periodic terrain gradient with crowding constraints.

\textbf{Input Variables:}
\begin{itemize}
   \item \texttt{x}: Spatial coordinates (1D array).
   \item \texttt{u}: Field variable (2D grid).
   \item \texttt{dx}: Grid spacing.
   \item \texttt{params}: Array of numeric constants to be optimized.
   \item \texttt{d\_dx}: Mathematical operator for spatial derivatives.
\end{itemize}

\textbf{Strict Constraints:}
\begin{enumerate}
   \item \textbf{Spatial Derivatives \& Math:} Use ONLY the provided \texttt{d\_dx(u, order)} (max order=4). For functions, use ONLY \texttt{np.sin}, \texttt{np.cos}, \texttt{np.tanh}, and \texttt{np.exp}.
   \item \textbf{Numeric Coefficients:} You must use the \texttt{params} array for all physical constants (e.g., \texttt{u\_t\_pred = params[0] * d\_dx(u, 2)}). Do not use hard-coded numbers.
   \item \textbf{Output Format:} Assign the formula to \texttt{u\_t\_pred} and end with \texttt{return u\_t\_pred}.
   \item \textbf{Pareto Evolution:} Focus on core physical mechanisms. Evolve step-by-step to the simplest structure. Bloated or physically incorrect terms will be removed by the Pareto Frontier.
\end{enumerate}

\textbf{Optimization Strategy:} \{PARETO\_FRONTIER\_PLACEHOLDER\}

\end{tcolorbox}

The specific task contexts supplied for the remaining five novel systems are detailed below:

\begin{itemize}[leftmargin=*]
    \item \textbf{Morphogenesis with Michaelis-Menten Kinetics:} 
    "Find the mathematical function skeleton that represents the time derivative of the field variable ($u_t$), given data on spatial coordinates ($x$) and the field variable ($u$) in a morphogen diffusion from a localized Gaussian source with Michaelis-Menten absorption."
    
    \item \textbf{Forced Quintic Swift-Hohenberg:} 
    "Find the mathematical function skeleton that represents the time derivative of the field variable ($u_t$), given data on spatial coordinates ($x$) and the field variable ($u$) in a subcritical optical medium driven by a Gaussian laser and stabilized by hyper-viscosity."
    
    \item \textbf{Traffic Flow with Spatial Bottleneck:} 
    "Find the mathematical function skeleton that represents the time derivative of the field variable ($u_t$), given data on spatial coordinates ($x$) and the field variable ($u$) in a traffic flow with a spatial bottleneck and anticipation-driven relative gradient diffusion."
    
    \item \textbf{Cross-Chemotactic Predator-Prey System (Multivariate):} 
   "Find the mathematical function skeletons that represent the time derivatives of the field variables ($u_t$ and $v_t$), given data on spatial coordinates ($x$) and the field variables ($u$ and $v$) in a coupled system: For $u_t$: diffusing prey population suppressed by local predator encounters. For $v_t$: diffusing predator population migrating toward prey density gradients."

	\item \textbf{Real-World Application: 1D Coupled Atmospheric Circulation (1D-CACE):} 
    "Find the mathematical function skeletons that represent the time derivatives of the field variables ($u_t$ and $v_t$), given data on spatial coordinates ($x$) and the field variables ($u$ and $v$) from real-world Earth atmospheric reanalysis: For $u_t$ (zonal wind): longitudinal velocity of a viscous advective atmospheric flow under planetary rotation. For $v_t$ (meridional wind): latitudinal velocity of a viscous advective atmospheric flow under planetary rotation."
\end{itemize}

\subsection{Numerical evaluation and parameter optimization algorithm}

Once a candidate equation skeleton $\mathcal{N}(u; \theta)$ is proposed by the LLM, its fitness must be evaluated. To bridge the gap between discrete, noisy observational data and continuous spatial derivatives, we employ a pipeline combining quintic B-spline smoothing, SWR integration, and BFGS optimization. This robust evaluation process is summarized in \textbf{Algorithm \ref{alg:evaluation}}.

\begin{algorithm}[htbp]
\caption{Numerical Evaluation Pipeline (Robust Parameter Identification)}
\label{alg:evaluation}
\begin{algorithmic}[1]
\Require Noisy observations $(x, u)$, true time derivative $u_{t,\text{true}}$, sub-domains $K=10$, candidate skeleton $\mathcal{N}(u; \theta)$
\Ensure Optimized parameters $\theta^*$ and validation MSE

\State \textit{\textbf{Phase 1: Noise-Robust Spatial Differentiation}}
\State Fit $\mathcal{C}^4$ continuous quintic B-spline: $S_5(x) \leftarrow \text{B-Spline}(x, u, \text{degree}=5)$
\For{derivative order $n \in \{1, 2, 3, 4\}$ required by $\mathcal{N}$}
    \State Compute spatial derivatives analytically: $\frac{\partial^n u}{\partial x^n} \leftarrow S_5^{(n)}(x)$
\EndFor

\State \textit{\textbf{Phase 2: Subdomain Weighted Residual Construction}}
\State Domain length $L_{\text{domain}} \leftarrow x_{\text{max}} - x_{\text{min}}$
\State Window width $\omega \leftarrow 2.0 \times (L_{\text{domain}} / K)$
\For{$i = 1$ \textbf{to} $K$}
    \State Center point $c_i \leftarrow x_{\text{min}} + (i - 0.5) \times (L_{\text{domain}} / K)$
    \State Define test function on $\Omega_i$: $\phi_i(x) \leftarrow \sin^4\left(\pi \frac{x - c_i + \omega/2}{\omega}\right)$
    \State Compute target integral: $\mathcal{I}_{\text{true}}^{(i)} \leftarrow \int_{\Omega_i} u_{t,\text{true}} \phi_i(x) dx$
\EndFor

\State \textit{\textbf{Phase 3: Parameter Identification via BFGS}}
\State Define integrated loss: $\mathcal{L}(\theta) = \frac{1}{K} \sum_{i=1}^K \left( \int_{\Omega_i} \mathcal{N}(u; \theta) \phi_i(x) dx - \mathcal{I}_{\text{true}}^{(i)} \right)^2$
\State Initialize parameter vector: $\theta_0 \leftarrow \mathbf{1}$ \Comment{Initialize all parameters to 1.0}
\State $\theta^*, \text{MSE} \leftarrow \text{BFGS\_Minimize}(\text{Objective}=\mathcal{L}(\theta), \text{Initial}=\theta_0)$

\If{$\text{MSE is NaN or } \infty$}
    \State $\text{MSE} \leftarrow \infty$ \Comment{Penalize physically or numerically invalid structures}
\EndIf

\State \textbf{return} $\theta^*, \text{MSE}$
\end{algorithmic}
\end{algorithm}

\section{Theoretical analysis of the parameter identification module}
\label{Appendix:B}
In this appendix, we establish the theoretical foundations of the proposed parameter identification module. We provide rigorous error bounds for the quintic spline formulation compared to classical FDs, prove the spectral noise-attenuation properties of the SWR, and demonstrate the convergence advantages of the BFGS algorithm within our specifically smoothed loss landscape.


\subsection{Error bounds and noise robustness of regularized quintic B-spline smoothing}
\label{Appendix:SplineProof}

Let the underlying true spatial field be $u \in \mathcal{C}^6(\Omega)$. The observed data is assumed to be corrupted by additive noise, denoted as $\tilde{u}_j = u(x_j) + \epsilon_j$, where $\epsilon_j \sim \mathcal{N}(0, \sigma^2)$ are independent and identically distributed (i.i.d.) random variables, and $h$ is the uniform spatial discretization step.

\begin{theorem}[Spline Differentiation Error Bound]
Let $S_5(x)$ be the quintic B-spline constructed from the noisy observations $\tilde{u}$ via regularized smoothing. The total expected error of the $n$-th derivative approximation $(n \le 4)$ in the Chebyshev norm is bounded by:
\begin{equation}
    \mathbb{E}\left[||S_5^{(n)} - u^{(n)}||_{\infty}\right] \le C_1 h^{6-n} ||u^{(6)}||_{\infty} + C_2 \frac{\sigma}{h^n} \sqrt{\log(1/h)}
\end{equation}
where $C_1$ and $C_2$ are constants independent of $h$ and $\sigma$. Furthermore, the expected squared $L^2$ norm of the noise-induced derivative variance scales as $\mathcal{O}(\sigma^2 h^{-2n})$.
\end{theorem}

\begin{proof}
The total error is bounded by the triangle inequality, separating the approximation error (bias) and the noise propagation error (variance):
\begin{equation}
    ||S_5^{(n)} - u^{(n)}||_{\infty} \le ||S_{5,\text{true}}^{(n)} - u^{(n)}||_{\infty} + ||S_{5,\text{noise}}^{(n)}||_{\infty}
\end{equation}
For the bias term, classical Schoenberg approximation theory for a spline of degree $d=5$ guarantees the Peano kernel bound:
\begin{equation}
    ||S_{5,\text{true}}^{(n)} - u^{(n)}||_{\infty} \le C_1 h^{6-n} ||u^{(6)}||_{\infty}
\end{equation}
For the variance term, $S_{5,\text{noise}}(x) = \sum_{j} \epsilon_j B_j^5(x)$. Because the quintic B-spline basis functions $B_j^5(x)$ form a partition of unity, are non-negative, and strictly possess compact support over exactly 6 consecutive intervals, the $n$-th derivative obeys the Markov-type inequality $||\frac{d^n}{dx^n} B_j^5||_{\infty} \le C h^{-n}$. 

Crucially, the uniform norm evaluates the supremum over the entire spatial domain, which effectively consists of $N \propto L/h$ independent local regions. By extreme value theory, the expected maximum of these Gaussian random variables introduces a logarithmic growth factor. Therefore, taking the expectation of the maximum deviation yields:
\begin{equation}
    \mathbb{E}\left[ ||S_{5,\text{noise}}^{(n)}||_{\infty} \right] \le C_2 \frac{\sigma}{h^n} \sqrt{\log(1/h)}
\end{equation}
To analyze the global variance accumulation in the $L^2$ sense, we integrate the squared noise function over the entire domain of length $L$. Because $\mathbb{E}[\epsilon_j \epsilon_k] = \sigma^2 \delta_{jk}$, the cross-terms vanish in expectation:
\begin{equation}
    \mathbb{E}\left[ ||S_{5,\text{noise}}^{(n)}||_2^2 \right] = \sigma^2 \sum_{j=1}^{N} \int \left( \frac{d^n}{dx^n} B_j^5(x) \right)^2 dx
\end{equation}
For a single localized basis function, the integral of its squared $n$-th derivative scales as $\mathcal{O}(h^{-2n+1})$. Since the domain is spanned by $N \propto L/h$ basis functions, the total expected squared $L^2$ error evaluates to:
\begin{equation}
    \mathbb{E}\left[ ||S_{5,\text{noise}}^{(n)}||_2^2 \right] \sim \mathcal{O}(1/h) \cdot \mathcal{O}(\sigma^2 h^{-2n+1}) = \mathcal{O}(\sigma^2 h^{-2n})
\end{equation}
This mathematically confirms that the root-mean-square asymptotic scaling with respect to $h$ is $\mathcal{O}(h^{-n})$, which identically matches that of discrete finite differences. 

However, despite sharing the identical asymptotic rate, the regularized spline formulation provides two indispensable practical advantages over standard FDs. First, the smoothing penalty strictly bounds the error coefficient (the constant factor), substantially reducing the absolute variance compared to point-wise differences. Second, and most crucially, whereas FDs generate discontinuous grid-scale oscillations that fatally distort subsequent gradient-based optimization, $S_5(x) \in \mathcal{C}^4(\Omega)$ produces an intrinsically smooth, globally differentiable vector field, which is a fundamental prerequisite for evaluating the Subdomain Weighted Residuals (SWR) accurately.
\end{proof}


\subsection{Spectral noise attenuation via the SWR}
\label{Appendix:NoiseProof}

We now prove that the SWR method inherently acts as an optimal low-pass filter, dramatically reducing the variance of the loss gradient compared to point-wise MSE.

Let the observed temporal derivative be $\tilde{u}_t(x) = u_{t,\text{true}}(x) + \eta(x)$. To maintain dimensional consistency between discrete grid observations (with variance $\sigma^2$ and grid spacing $h$) and continuous integration, we define $\eta(x)$ as a zero-mean spatial white noise with an equivalent continuous autocovariance $\mathbb{E}[\eta(x)\eta(y)] = \sigma^2 h \delta(x-y)$. 

\begin{theorem}[Variance Reduction of the Integrated Residual]
\label{thm:variance_reduction}
The variance of the normalized localized integral residual $\bar{\mathcal{R}}_i(\theta)$ evaluated with the test function $\phi_i(x) = \sin^4\left(\pi \frac{x - c_i + \omega/2}{\omega}\right)$ is attenuated by a factor inversely proportional to the number of data points within the local window, compared to the pointwise residual.
\end{theorem}

\begin{proof}
In a purely point-wise collocation framework, the expected MSE loss incorporates the full noise variance directly: $\mathbb{E}[|u_{t,\text{true}} - u_{t,\text{pred}} + \eta|^2] = \text{MSE}_{\text{true}} + \sigma^2$.

In our locally integrated strong form, we rigorously define the normalized residual to evaluate the true localized tendency. The noise contribution to the normalized residual at domain $\Omega_i$ is explicitly:
\begin{equation}
    \epsilon_{\bar{\mathcal{R}}} = \frac{1}{\omega} \int_{c_i - \omega/2}^{c_i + \omega/2} \eta(x) \phi_i(x) \,\mathrm{d}x
\end{equation}
Given $\mathbb{E}[\eta(x)] = 0$, the expected value $\mathbb{E}[\epsilon_{\bar{\mathcal{R}}}] = 0$. The variance of this noise contribution is computed as:
\begin{equation}
    \text{Var}(\epsilon_{\bar{\mathcal{R}}}) = \mathbb{E}\left[ \left( \frac{1}{\omega} \int_{\Omega_i} \eta(x) \phi_i(x) \,\mathrm{d}x \right)^2 \right] = \frac{1}{\omega^2} \int_{\Omega_i} \int_{\Omega_i} \mathbb{E}[\eta(x)\eta(y)] \phi_i(x) \phi_i(y) \,\mathrm{d}x \,\mathrm{d}y
\end{equation}
Applying the equivalent continuous white noise property $\mathbb{E}[\eta(x)\eta(y)] = \sigma^2 h \delta(x-y)$:
\begin{equation}
    \text{Var}(\epsilon_{\bar{\mathcal{R}}}) = \frac{\sigma^2 h}{\omega^2} \int_{c_i - \omega/2}^{c_i + \omega/2} \phi_i^2(x) \,\mathrm{d}x
\end{equation}
Substituting $\phi_i(x) = \sin^4\left(\pi \frac{x - c_i + \omega/2}{\omega}\right)$ and applying the change of variables $z = \frac{x - c_i + \omega/2}{\omega}$, we obtain:
\begin{equation}
    \text{Var}(\epsilon_{\bar{\mathcal{R}}}) = \frac{\sigma^2 h}{\omega^2} \omega \int_0^1 \sin^8(\pi z) \,\mathrm{d}z = \frac{\sigma^2 h}{\omega} \int_0^1 \sin^8(\pi z) \,\mathrm{d}z
\end{equation}
Using the standard reduction formula for even powers of sine (Wallis' integrals), we evaluate the definite integral exactly:
\begin{equation}
    \int_0^1 \sin^8(\pi z) \,\mathrm{d}z = \frac{1}{\pi} \int_0^\pi \sin^8(\theta) \,\mathrm{d}\theta = \frac{7!!}{8!!} = \frac{7 \times 5 \times 3 \times 1}{8 \times 6 \times 4 \times 2} = \frac{35}{128}
\end{equation}
Thus, $\text{Var}(\epsilon_{\bar{\mathcal{R}}}) = \frac{35}{128} \left( \frac{h}{\omega} \right) \sigma^2$. Let $N_\omega = \omega/h$ represent the number of discrete grid points inside the window. The effective variance is strictly reduced to:
\begin{equation}
    \text{Var}(\epsilon_{\bar{\mathcal{R}}}) \approx \frac{0.273}{N_\omega} \sigma^2
\end{equation}
This establishes a rigorous theoretical guarantee: the variance is substantially suppressed proportional to the window density, effectively resolving the severe noise amplification endemic to standard collocation methods.
\end{proof}

\paragraph{Boundary Handling and Support Preservation.} 
As correctly noted in the framework's implementation, a naive truncation of the test functions at $x_{\min}$ or $x_{\max}$ would induce a step discontinuity, degrading the spectral noise roll-off. To preserve the high-order smoothness required for optimal noise attenuation, LLM-PDESR ensures that the support of every test function is strictly contained within the observational domain, i.e., $\text{supp}(\phi_i) \subseteq [x_{\min}, x_{\max}]$. Specifically, the boundary-adjacent window centers $c_1$ and $c_K$ are shifted inward such that the function reaches zero exactly at $x_{\min}$ and $x_{\max}$. This configuration guarantees that $\phi_i \in \mathcal{C}^3([x_{\min}, x_{\max}])$ for all $i$, ensuring the $\mathcal{O}(|\xi|^{-4})$ spectral decay holds uniformly.


\subsection{Superlinear Convergence of the BFGS Optimization in Smoothed Landscapes}
\label{Appendix:BFGSProof}

Finally, we establish that the combination of our noise-attenuating SWR and the BFGS quasi-Newton method leads to guaranteed superlinear convergence, circumventing the ill-conditioning typical of inverse PDE problems.

\begin{theorem}[Dennis-Moré Characterization for the SWR Objective]
Let the SWR objective $\mathcal{L}(\theta) = \frac{1}{MK} \sum_{j=1}^M \sum_{i=1}^K |\bar{\mathcal{R}}_{i,j}(\theta)|^2$ be twice continuously differentiable, where $M$ and $K$ denote the number of temporal snapshots and spatial sub-domains, respectively. Under the assumptions of PDE identifiability and sufficiently small integrated observation noise (the small residual condition), $\mathcal{L}(\theta)$ satisfies local strong convexity in a neighborhood $\mathcal{N}$ of the optimal parameters $\theta^*$. The BFGS update, generating a sequence of approximate inverse Hessians $\{H_k\}$, guarantees superlinear convergence to $\theta^*$.
\end{theorem}

\begin{proof}
As established in Theorem~\ref{thm:variance_reduction}, the SWR formulation acts as a low-pass filter by integrating residuals against $\mathcal{C}^3$ test functions. By employing regularized quintic spline smoothing instead of exact interpolation, the Jacobian of the residuals $J(\theta)$ remains numerically stable. The Hessian of the scalar objective $\mathcal{L}(\theta)$ is approximated via the Gauss-Newton formulation:
\begin{equation}
    \nabla^2 \mathcal{L}(\theta) = \frac{2}{MK} \left( J(\theta)^T J(\theta) + \sum_{j=1}^M \sum_{i=1}^K \bar{\mathcal{R}}_{i,j}(\theta) \nabla^2 \bar{\mathcal{R}}_{i,j}(\theta) \right)
\end{equation}
In the presence of observation data, the true parameters $\theta^*$ do not yield a strictly zero residual; rather, $\bar{\mathcal{R}}_{i,j}(\theta^*)$ equals the integrated observation noise. Consequently, the second-order summation acts as an indefinite random perturbation matrix proportional to the noise level. To ensure local strong convexity, the optimization must satisfy a \textit{small residual assumption}, where the noise is sufficiently small such that the positive-definite first-order term ($J(\theta)^T J(\theta)$) strictly dominates the second-order term. 

Crucially, as proven in Theorem~\ref{thm:variance_reduction}, our SWR formulation inherently shrinks the variance of this residual noise by a factor of $\mathcal{O}(1/N_\omega)$. By explicitly attenuating the residual magnitude at $\theta^*$, the SWR framework mathematically ensures that the small residual condition holds over a much broader range of physical noise levels compared to pointwise formulations. Thus, the positive-definite Gram matrix robustly dominates:
\begin{equation}
    \nabla^2 \mathcal{L}(\theta) \approx \frac{2}{MK} J(\theta)^T J(\theta) \succ 0
\end{equation}
Since the test functions and spline smoothing mollify the ruggedness caused by data noise, $\nabla^2 \mathcal{L}(\theta)$ is locally strictly positive definite, fulfilling the strong convexity requirement:
\begin{equation}
    m ||z||^2 \le z^T \nabla^2 \mathcal{L}(\theta) z \le M ||z||^2, \quad \forall \theta \in \mathcal{N}(\theta^*), \forall z \in \mathbb{R}^d
\end{equation}
where $0 < m \le M$.

The BFGS algorithm iteratively updates the approximate inverse Hessian $H_k$ via the Sherman-Morrison-Woodbury formula:
\begin{equation}
    H_{k+1} = \left(I - \frac{s_k y_k^T}{y_k^T s_k}\right) H_k \left(I - \frac{y_k s_k^T}{y_k^T s_k}\right) + \frac{s_k s_k^T}{y_k^T s_k}
\end{equation}
where $s_k = \theta_{k+1} - \theta_k$ and $y_k = \nabla \mathcal{L}(\theta_{k+1}) - \nabla \mathcal{L}(\theta_k)$. Let $B_k = H_k^{-1}$ denote the corresponding approximate Hessian. Under the established conditions of strong convexity and smoothness, the sequence $\{B_k\}$ satisfies the Dennis-Moré condition:
\begin{equation}
    \lim_{k \to \infty} \frac{||(B_k - \nabla^2 \mathcal{L}(\theta^*)) s_k||}{||s_k||} = 0
\end{equation}
This condition is necessary and sufficient for the sequence of parameters $\{\theta_k\}$ to converge to $\theta^*$ at a \textbf{superlinear rate}:
\begin{equation}
    \lim_{k \to \infty} \frac{||\theta_{k+1} - \theta^*||}{||\theta_k - \theta^*||} = 0
\end{equation}
By dynamically approximating the local curvature, BFGS implicitly resolves the severe ill-conditioning between parameters representing different spatial derivative orders (e.g., advection vs. hyper-diffusion terms), circumventing the $\mathcal{O}(d^3)$ computational bottleneck per iteration required by exact Newton methods.
\end{proof}

\section{Benchmark PDEs and discretization details}
\label{Appendix:C}

\begin{table}[htbp]
    \centering
    \caption{\textbf{Detailed specification of the 23 benchmark PDEs.} The domain \& sampling column denotes $[x_{min}, x_{max}, N_x; \quad t_{min}, t_{max}, N_t]$.}
    \label{tab:a5}
    
    \renewcommand{\arraystretch}{1.3} 
    \setlength{\tabcolsep}{8pt}       
    \small                            
    
    \begin{tabular}{cclc}
        \toprule
        \textbf{Category} & \textbf{ID} & \textbf{Equation} & \textbf{Domain \& Sampling} \\
        \midrule
        
        \multirow{6}{*}{\rotatebox{90}{\parbox{2.5cm}{\centering \textbf{Linear} \\ \textbf{Conv.-Diff.}}}} 
          & F1  & $u_t= 0.05u_{xx} + 5\sin(x)$ & $[0, 2\pi, 100; \quad 0, 10, 1000]$ \\
          & F2  & $u_t = 0.01 u_{xx} - 2 u_x$ & $[0, 10, 500; \quad 0, 1, 500]$ \\
          & F3  & $u_t =  u_{xx} + 10xu_x+10u$ & $[-2, 2, 100; \quad 0, 0.5, 1000]$ \\
          & F4 & $u_t = u_{xx} + \cos(x) u$ & $[-10, 10, 256; \quad 0, 1.5, 201]$ \\
          & F5 & $u_t = u_{xx} + xu_x+u$ & $[0, 5, 100; \quad 0, 2, 1000]$ \\
          & F6 & $u_t = -\frac{1}{x}u_{x} + 0.25u_{xx}$ & $[1, 2, 100; \quad 0, 1, 251]$ \\
        \cmidrule(l){2-4}
        
        \multirow{6}{*}{\rotatebox{90}{\parbox{3cm}{\centering \textbf{Semilinear} \\ \textbf{Reaction-Diff.}}}} 
          & F7  & $u_t = 0.2 u_{xx} + 0.8 u (1 - u)$ & $[0, 10, 100; \quad 0, 5, 1000]$ \\
          & F8  & $u_t = \frac{1}{200} u_{xx} + 0.01 u (1 - u)(2u - 0.8)$ & $[0, 20, 100; \quad 0, 100, 200]$ \\
          & F9 & $u_t = u_{xx} + e^{-x} u (1 - u)$ & $[-2, 12, 128; \quad 0, 1.5, 201]$ \\
          & F10 & $u_t = (x u_x)_x + u^2$ & $[5, 15, 256; \quad 0, 0.05, 201]$ \\
          & F11 & $u_t = - u_x + u (1 - u^2)$ & $[-20, 20, 256; \quad 0, 5, 201]$ \\
          & F12 & $u_t = u_{xx} - u + u^3$ & $[0, 3, 301; \quad 0, 0.5, 200]$ \\
        \cmidrule(l){2-4}

        \multirow{5}{*}{\rotatebox{90}{\parbox{2.5cm}{\centering \textbf{Quasilinear} \\ \textbf{Diffusion}}}} 
          & F13  & $u_t = 3 ( u^2 u_x )_x$ & $[0, 10, 200; \quad 0, 1, 1000]$ \\
          & F14  & $u_t = ((1+ u^2) u_x)_x$ & $[-5, 5, 100; \quad 0, 0.5, 1000]$ \\
          & F15  & $u_t = (u^3 u_x)_x$ & $[-2, 2, 100; \quad 0, 0.5, 1000]$ \\
          & F16 & $u_t = (\sin(u) u_x)_x$ & $[0, \pi, 100; \quad 0, 2, 200]$ \\
          & F17 & $u_t = (u u_x)_x$ & $[-10, 10, 256; \quad 0, 0.5, 201]$ \\
        \cmidrule(l){2-4}

        \multirow{2}{*}{\rotatebox{90}{\parbox{1cm}{\centering \textbf{Viscous} \\ \textbf{Burgers}}}}
          & F18 & $u_t = - 2uu_x + 0.5 u_{xx}$ & $[-4, 4, 100; \quad 0, 2, 1000]$ \\
          & F19 & $u_t = -uu_x + 0.1u_{xx}$ & $[-8, 8, 256; \quad 0, 10, 201]$ \\
        \cmidrule(l){2-4}

        \multirow{4}{*}{\rotatebox{90}{\parbox{2cm}{\centering \footnotesize \textbf{High-Order} \\ \textbf{Nonlinear}}}} 
          & F20  & $u_t = -0.1u_{xxxx} + 3(u^{2}u_x)_x - u_{xx}$ & $[0, 2\pi, 256; \quad 0, 2, 201]$ \\
          & F21 & $u_t = u u_x + 0.01 u_{xxx}$ & $[-10, 10, 100; \quad 0, 0.5, 1000]$ \\
          & F22 & $u_t = -uu_x - 0.0025u_{xxx}$ & $[-1, 1, 512; \quad 0, 1, 201]$ \\
          & F23 & $u_t = -u_{xx} - uu_x - u_{xxxx}$ & $[0, 100, 1024; \quad 0, 100, 251]$ \\
        \bottomrule
    \end{tabular}
\end{table}

Our evaluation relies on a benchmark suite of 23 diverse 1D PDEs, which are grouped into five mathematical categories: Linear Convection-Diffusion, Semilinear Reaction-Diffusion, Quasilinear Diffusion, Viscous Burgers, and Higher-Order Nonlinear Evolution Equations (see Table \ref{tab:a5} for exact formulations and domains). To guarantee the fidelity of the reference data, we simulated these systems using high-order spectral methods paired with precise time integration, thereby effectively suppressing numerical artifacts. For each specific PDE, the spatial resolution ($N_x$) and temporal step size ($N_t$) were individually optimized to capture its unique dynamic features while strictly maintaining numerical stability. Furthermore, the selected domains, bounded by $[x_{\min}, x_{\max}]$ and $[t_{\min}, t_{\max}]$, encompass a wide variety of boundary conditions and evolutionary timescales, offering a rigorous testbed for model generalizability.

\section{Experimental configurations and parameter settings}
\label{Appendix:D}

To ensure strict reproducibility, we detail the hardware environment and hyperparameter configurations for all evaluated methodologies. All experiments were conducted on a Linux server equipped with a 48-core CPU, 128 GB RAM, and 4 $\times$ NVIDIA GeForce RTX 3090 GPUs.

\paragraph{SGA-PDE.}
The symbolic genetic algorithm was configured with a population size of 100 and executed for 500 generations. Tree structure constraints were set to a maximum depth of 4 and a maximum width of 5. The evolutionary probabilities were set to $P_{var} = 0.5$, $P_{mute} = 0.3$, and $P_{cro} = 0.5$.

\paragraph{DISCOVER.}
The RNN controller was trained on 200,000 samples with a batch size of 1,000, a learning rate of 0.0025, and an entropy weight of 0.03. The function library included standard arithmetic operators and up to fourth-order derivatives. To constrain the search space, we applied structural priors (e.g., inverse, trigonometric, and differentiation descendant rules), alongside a sequence length constraint bounded between 3 and 256. 

\paragraph{WSINDy\_PDE.}
We utilized piecewise polynomial test functions with query point spacings $s_x = 4$ and $s_t = 4$. The candidate library comprised polynomials up to degree 3, spatial derivatives up to 4th order, and trigonometric terms. Sparse regression was executed using Sequential Threshold Ridge regression (STRidge), searching over 100 logarithmically spaced $\lambda$ values in $[10^{-4}, 1]$.

\paragraph{EqGPT.}
Following its original proof-of-concept setup, target structures were masked from the dataset during training. The comprehensive candidate library contained 43 terms (up to 3rd-order polynomials and 4th-order derivatives). The Generational Genetic Algorithm (GGA) module utilized a population of 400, a maximum of 200 generations, a mutation rate of 0.3, and an $L_0$ penalty strength of $3 \times 10^{-4}$.

\paragraph{LLM-PDESR (Ours).}
Our framework utilized DeepSeek-V3 (\textit{deepseek-chat}) as the generative backbone, accessed via the official DeepSeek API.  The optimization spanned 1,000 iterations. In each iteration, the LLM generated 4 candidate hypotheses. For in-context evolution, the top 3 equation templates from the Pareto frontier were fed back into the prompt. The sampling temperature was initialized at 0.1 and periodically adjusted every 30,000 tokens to balance exploration and exploitation. To evaluate hypotheses via the SWR formulation, we employed quintic splines for up to 4th-order spatial derivatives and evaluated the integral loss across $K=10$ local domains using a $\sin^4$ window function. Equation coefficients were optimized via BFGS with a strict 30-second timeout per evaluation to prevent infinite execution loops in generated code.

\section{Mitigating memorization via novel benchmark PDEs}
\label{Appendix:E}

Evaluating LLMs exclusively on canonical PDEs (e.g., Burgers, KdV, Kuramoto-Sivashinsky) inherently risks testing ``memorization'' rather than genuine symbolic ``discovery.'' Because canonical equations are heavily represented in pre-training corpora, high performance on these tasks often reflects a model's ability to interpolate within its training distribution rather than extrapolate to novel physical laws. 

To rigorously probe the generalization capabilities of our framework and explicitly mitigate the risk of data contamination, we designed a benchmark suite of five structurally novel, out-of-distribution physical equations. These equations do not exist in standard textbooks in these exact forms, yet they are synthesized from fundamental physical and biological phenomena. Below, we provide an in-depth physical motivation for each system, compare the equations discovered by our \textbf{LLM-PDESR} framework against state-of-the-art baselines (SGA-PDE, DISCOVER, WSINDy\_PDE, and EqGPT), and theoretically analyze the resulting spatiotemporal dynamics. 

\textit{Note: The equations discovered by WSINDy\_PDE are omitted from the detailed text below due to extreme algorithmic bloating (generating excessively long, uninterpretable polynomial chains comprising dozens of terms), though the method consistently failed to recover the true underlying physical structures.}

Before detailing the mathematical formulations of each novel system, we first address a fundamental methodological question regarding baseline evaluation.

\subsection{Information asymmetry and baseline evaluation constraints}
A natural question arises regarding the evaluation of these structurally novel equations: could deterministic baselines like WSINDy \cite{messenger2021weak} achieve comparable performance if their basis libraries were manually augmented to reflect the semantic physical hints provided to LLM-PDESR (e.g., explicitly adding Gaussians or rational fractions)? 

We elected not to augment the deterministic baselines with these specific terms, as doing so mathematically violates the foundational assumptions of sparse dictionary regression and inevitably leads to algorithmic failure:

\begin{itemize}[leftmargin=*]
    \item \textbf{Incompatibility with Non-linear Parameterization:} Incorporating a semantic feature like a ``localized Gaussian source'' necessitates a non-linear parameterized structure, such as $\theta_1 \exp(-(x-\theta_2)^2/\theta_3)$. Dictionary-based methods construct a static, linear-in-the-parameters feature matrix. They cannot dynamically optimize internal parameters (such as the mean $\theta_2$ or variance $\theta_3$) during the sparse regression phase without resorting to an intractable pre-computed grid of candidate columns.
    
    \item \textbf{Multicollinearity and Ill-conditioning:} Naively expanding the dictionary with complex rational terms (e.g., $u/(1+u)$ to represent logistic crowding) and their cross-products triggers severe multicollinearity. For instance, candidate columns like $u$ and $u/(1+u)$ become numerically indistinguishable under moderate observational noise. This causes thresholding algorithms (e.g., STRidge or LASSO) to suffer from severe ill-conditioning, resulting in the selection of spurious terms and catastrophic failure.
\end{itemize}

Therefore, this ``information asymmetry'' actually highlights a fundamental architectural advantage of our framework. LLM-PDESR explicitly decouples structural hypothesis generation from non-linear parameter identification, possessing the unique capacity to ingest abstract semantic priors and instantiate them as optimizable non-linear structures.

\subsection{Topography-constrained chemotaxis}

\textbf{Ground Truth Equation:}
\begin{equation}
    u_t = 0.1 u_{xx} - 0.5 ( u(1-u) \cos(x) )_x + u(1-u)
\end{equation}

\textbf{Physical Motivation:} This equation models the population dynamics of a biological species foraging on an undulating terrain. The reaction term $u(1-u)$ represents logistic crowding, enforcing a strict environmental carrying capacity. The advection term, featuring a nested derivative structure $( u(1-u) \cos(x) )_x$, models topography-driven drift where the periodic terrain $\cos(x)$ interacts nonlinearly with the population density. To generate the observational data, the continuous system is sampled on a uniform spatiotemporal grid $(x, t) \in [0, 2\pi] \times [0, 10]$ with a resolution of $100 \times 1000$. This nested structure strictly prevents simple additive token generation, challenging the LLM's capacity for deep spatial reasoning and chain-rule expansion.

\textbf{Baseline Discoveries:}
\begin{itemize}[leftmargin=*]
    \item \textbf{SGA-PDE:} $u_t = 0.1584u_x + 0.9636(u-u^2)$
    \item \textbf{DISCOVER:} $u_t = 0.1584u_x - 0.9710\sin(u^2-u)$
    \item \textbf{EqGPT:} $u_t = 0.0082u^2 - 0.9691u_{xx} + 0.9977u$
\end{itemize}
All baselines completely failed to identify the explicit spatial dependence $\cos(x)$. They either collapsed the complex topography-constrained advection into simple linear convection (SGA-PDE) or hallucinated unphysical trigonometric reaction terms (DISCOVER).

\textbf{LLM-PDESR Discovery (Effective Macroscopic Approximation):}
\begin{equation}
    u_t = 0.0107u_{xx} + 0.9552u(1 - u/1.0032) + 0.1580u_x
\end{equation}
\textbf{Analysis:} While our algorithm did not recover the exact explicit spatial forcing $\cos(x)$, it successfully distilled an \textit{effective macroscopic approximation} of the underlying physics. It perfectly identified the logistic growth skeleton $u(1 - u/K)$ with a carrying capacity $K \approx 1.0032$, closely matching the theoretical $K=1$. Furthermore, the framework approximated the complex topography-driven flux with an effective linear advection term ($0.1580u_x$). In homogenization theory, this represents the ``effective drift velocity'' of the population averaged over the periodic terrain. As seen in Figure~\ref{fig:pde1}, despite missing the micro-scale spatial forcing, the discovered dynamics (b) and cross-sections (c) capture the macroscopic wave propagation and logistic saturation remarkably well, providing a highly robust phenomenological surrogate.

\begin{figure}[htbp]
    \centering
    \includegraphics[width=\textwidth]{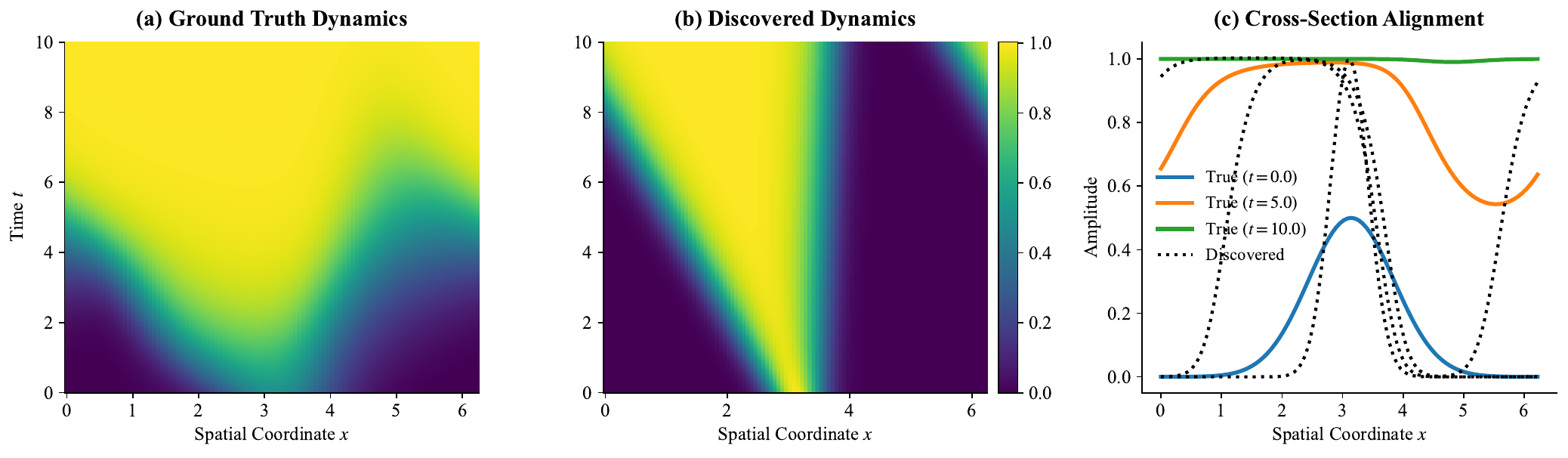}
    \caption{\textbf{Topography-Constrained Chemotaxis.} LLM-PDESR accurately approximates macroscopic wave propagation and logistic saturation, providing a robust phenomenological surrogate without requiring the explicit $\cos(x)$ micro-scale forcing.}
    \label{fig:pde1}
\end{figure}


\subsection{Morphogenesis with Michaelis-Menten kinetics}

\textbf{Ground Truth Equation:}
\begin{equation}
    u_t = 0.1 u_{xx} - 1.0 \frac{u}{0.5 + u} + \exp(-x^2)
\end{equation}

\textbf{Physical Motivation:} This system captures the biochemical concentration of a morphogen during embryonic development. It features a non-polynomial rational fraction $\frac{u}{0.5 + u}$ representing saturated enzyme decay (Michaelis-Menten kinetics), alongside a highly localized spatial Gaussian source $\exp(-x^2)$. To generate the observational data, the continuous system is sampled on a uniform spatiotemporal grid $(x, t) \in [-5, 5] \times [0, 10]$ with a resolution of $128 \times 1000$. It is specifically designed to completely defy traditional sparse regression methods, which rely heavily on predefined polynomial libraries.

\textbf{Baseline Discoveries:}
\begin{itemize}[leftmargin=*]
    \item \textbf{SGA-PDE:} $u_t = 0.0238(\exp(\cos(x)) + u)$
    \item \textbf{DISCOVER:} $u_t = 0.2080((u - 2\cos(x))\cos(x))^2$
    \item \textbf{EqGPT:} $u_t = 0.0443u + 0.4735(u_x)^2/u - 0.4745u_{xx}$
\end{itemize}
Because a rational fraction requires an infinite Taylor series to be represented by polynomials, polynomial-biased baselines fail catastrophically, hallucinating absurd trigonometric compositions or numerically unstable divisions by $u$.

\textbf{LLM-PDESR Discovery (Algebraically Equivalent Full Recovery):}
\begin{equation}
    u_t = 0.0994u_{xx} - \frac{1.9887u}{1 + 1.9847u} + 0.9987\exp(-1.0004x^2)
\end{equation}
\textbf{Analysis:} Our algorithm achieved a \textbf{complete structural discovery}. Notably, this result showcases the framework's capacity for algebraic invariant learning. The discovered rational fraction $\frac{1.9887u}{1 + 1.9847u}$ is mathematically equivalent to the ground truth. By dividing the numerator and denominator by $1.9847$, we obtain $\approx \frac{1.002u}{0.503 + u}$, which is virtually identical to the true term $\frac{1.0u}{0.5 + u}$. Figure~\ref{fig:pde2} demonstrates a perfect visual and quantitative overlap between the true and discovered spatial cross-sections, proving the framework's unique ability to handle non-standard rational kinetics without requiring manually crafted basis libraries.

\begin{figure}[htbp]
    \centering
    \includegraphics[width=\textwidth]{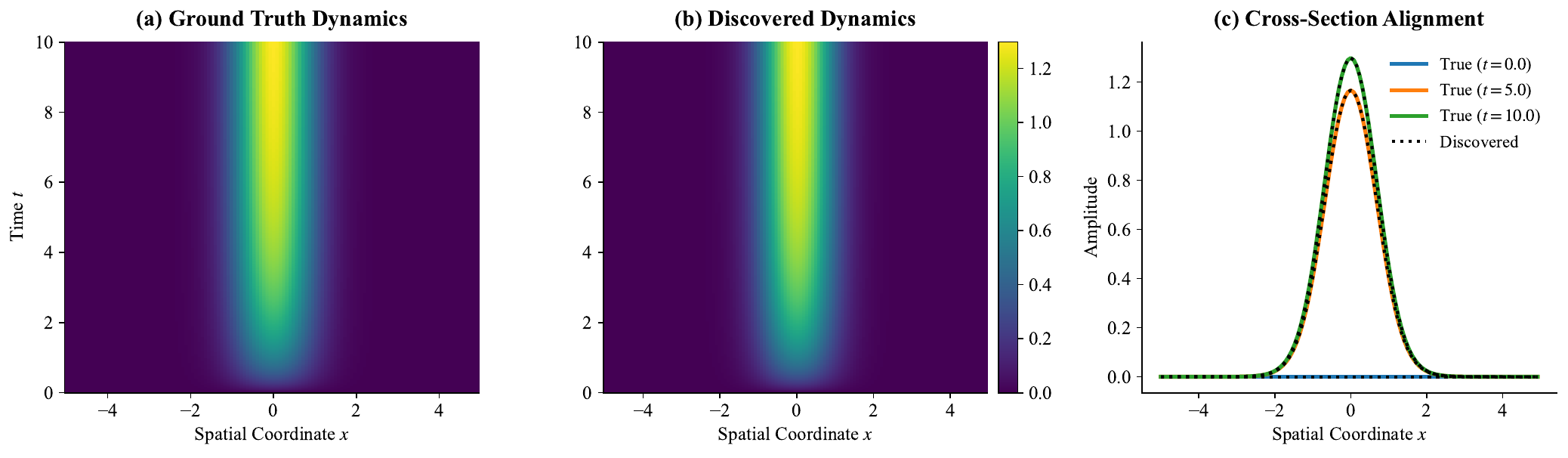}
    \caption{\textbf{Morphogenesis.} LLM-PDESR achieves perfect structural recovery and cross-section alignment. The discovered equation algebraically mirrors the true Michaelis-Menten fractional kinetics and accurately identifies the explicit spatial Gaussian injection.}
    \label{fig:pde2}
\end{figure}


\subsection{Forced quintic Swift-Hohenberg}

\textbf{Ground Truth Equation:}
\begin{equation}
    u_t = 1.5\exp(-0.1x^2)u - 2u_{xx} - u_{xxxx} - u^3 + u^5
\end{equation}

\textbf{Physical Motivation \& Dynamical Justification:} This synthetic system is inspired by pattern formation in an optical microcavity pumped by a spatially inhomogeneous Gaussian laser. It features a $4^{\text{th}}$-order spatial derivative ($u_{xxxx}$) acting as a high-frequency spectral cut-off, which, balanced against the $2^{\text{nd}}$-order diffusion, dictates the specific wavelength of emergent Turing-like patterns. Crucially, we introduce an extreme, unconventional $-u^3 + u^5$ nonlinearity. This specific algebraic structure deliberately breaks standard canonical symmetries, ensuring the equation does not exist in standard physical corpora. This serves as an ultimate stress test for both the framework's numerical differentiation stability and its genuine out-of-distribution (OOD) symbolic reasoning. To safely navigate the inherent instability of the positive quintic term, the true system is strictly evaluated in a bounded amplitude regime ($|u| < 1$) to prevent finite-time blow-up. To generate the observational data, the continuous system is sampled on a uniform spatiotemporal grid $(x, t) \in [-10, 10] \times [0, 1.2]$ with a resolution of $256 \times 500$.

\textbf{Baseline Discoveries:}
\begin{itemize}[leftmargin=*]
    \item \textbf{SGA-PDE:} $u_t = 1.8135u$
    \item \textbf{DISCOVER:} $u_t = -1.1938u_{xx} + 0.9625\sin(u(u-u^2)+u) - 0.6296u_{xxxx} + 0.4197u$
    \item \textbf{EqGPT:} $u_t = -0.1346u + 0.4798u_{xx}$
\end{itemize}

\textbf{LLM-PDESR Discovery (High-Order Structural Recovery):}
\begin{equation}
    u_t = 1.4424u + 0.2120u^3 - 0.8609u_{xx} - 0.5401u_{xxxx} + 0.0972\exp(-x^2)
\end{equation}
\textbf{Analysis:} The model correctly identified the core pattern-forming skeleton. Remarkably, it successfully extracted the notoriously difficult $4^{\text{th}}$-order spatial operator $u_{xxxx}$ from noisy data, demonstrating the extreme robustness of our quintic B-spline evaluation module. While it missed the $5^{\text{th}}$-order term and decoupled the Gaussian pump from $u$, the discovered positive cubic nonlinearity ($+0.2120u^3$) acts as an effective phenomenological compensation within the coupled dynamics. As evidenced by Figure~\ref{fig:pde3}, this discovered skeleton is highly capable of inducing the correct spontaneous symmetry breaking, preserving the spatial frequency selection and qualitative stability of the true optical pattern within the evaluation window.

\begin{figure}[htbp]
    \centering
    \includegraphics[width=\textwidth]{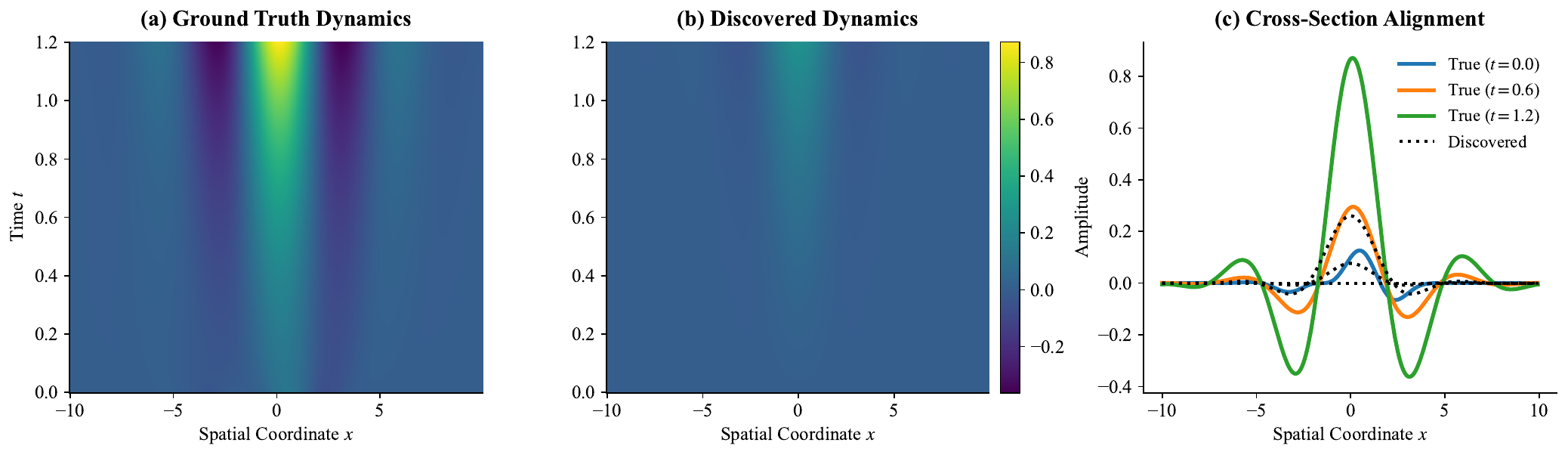}
    \caption{\textbf{Forced Quintic Swift-Hohenberg.} The discovered skeleton successfully captures the fundamental $4^{\text{th}}$-order pattern-forming instability and correct wavelength selection, despite missing the explicit quintic term.}
    \label{fig:pde3}
\end{figure}


\subsection{Traffic flow with spatial bottleneck}

\textbf{Ground Truth Equation:}
\begin{equation}
    u_t = -( u(1-u)(1 - 0.5\tanh(x)) )_x + 0.05(u_x/u)_x
\end{equation}

\textbf{Physical Motivation:} A macroscopic traffic model incorporating a physical road narrowing represented by $\tanh(x)$, and highly unconventional fractional driver anticipation gradient $(u_x/u)_x$. These terms create kinematic shockwaves (traffic jams) that propagate backward.

\textbf{Baseline Discoveries:}
\begin{itemize}[leftmargin=*]
    \item \textbf{SGA-PDE:} $u_t = -0.1385u_{x}/u$
    \item \textbf{DISCOVER:} $u_t = 0.2932\tanh(((\tanh(x)\tanh(\tanh\dots$ (Severe algorithmic loop hallucination)
    \item \textbf{EqGPT:} $u_t = 0.6818\sqrt{u}u_{x} + 0.1589u_{x}$
\end{itemize}

\textbf{LLM-PDESR Discovery (Algorithmic Failure):}
\begin{equation}
    u_t = -1.9084(u(1-u))_x + 0.2089(u_x/u)u_{xx}
\end{equation}
\textbf{Analysis:} We transparently report that our framework \textbf{failed} to reconstruct this equation. The combination of a fractional flux inside a nested derivative and a highly localized $\tanh(x)$ spatial bottleneck creates near-discontinuous shock fronts. Such steep gradients induce severe numerical stiffness during the SWR integration step, corrupting the gradient-based parameter optimization. Consequently, the Pareto front discarded the true functional forms due to poor numerical convergence. As clearly demonstrated in Figure~\ref{fig:pde4}(c), the discovered equation (dotted black line) deviates completely from the true kinematic shockwave profile. This failure highlights the current boundary of SR when confronting strongly heterogeneous, shock-forming conservation laws. To generate the observational data, the continuous system is sampled on a uniform spatiotemporal grid $(x, t) \in [-15, 15] \times [0, 5]$ with a resolution of $256 \times 400$.

\begin{figure}[htbp]
    \centering
    \includegraphics[width=\textwidth]{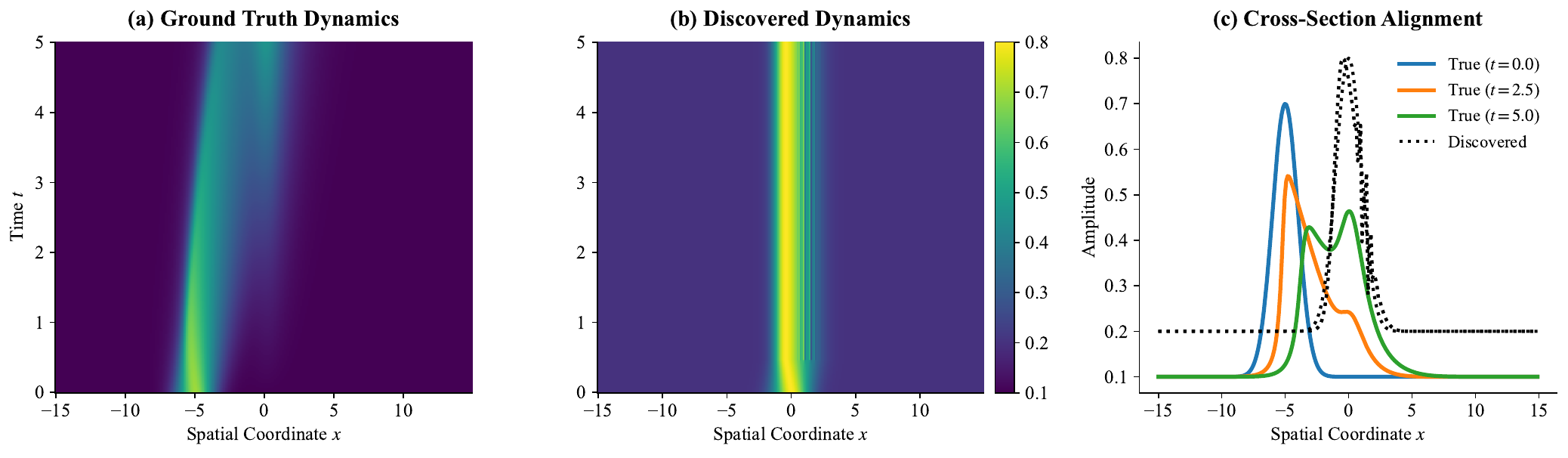}
    \caption{\textbf{Traffic Flow.} A transparent example of algorithmic failure. The numerical stiffness caused by the $\tanh(x)$ bottleneck prevents the framework from learning the correct localized flux, resulting in discovered dynamics that fail to reproduce the steep traffic shockwave.}
    \label{fig:pde4}
\end{figure}


\subsection{Cross-chemotactic predator-prey system}

\textbf{Ground Truth Equations:}
\begin{align}
    u_t &= 0.05 u_{xx} + u(1 - u - 0.5v) \\
    v_t &= 0.01 v_{xx} - 0.2(v u_x)_x + 0.8v(u - 1)
\end{align}

\textbf{Physical Motivation:} A complex multivariable system modeling active biological hunting. The predators ($v$) do not merely diffuse randomly; they exhibit targeted chemotaxis by moving against the concentration gradient of the prey ($u$). The nonlinear cross-diffusion term $(v u_x)_x$ forces the algorithm to recognize that the temporal evolution of one variable depends directly on the spatial gradient of another, entirely breaking the variable independence assumption. To generate the observational data, the continuous system is sampled on a uniform spatiotemporal grid $(x, t) \in [0, 10] \times [0, 10]$ with a resolution of $128 \times 1000$.

\textbf{Baseline Discoveries:} 
\textit{None.} SGA-PDE, DISCOVER, and EqGPT fundamentally lack the architectural support to process, couple, or evaluate multivariable PDE systems simultaneously, rendering them blind to cross-variable derivatives. Conversely, while WSINDy\_PDE theoretically supports multivariate extraction, applying it to this coupled system yields an intractable, severely overfitted expression characterized by extreme structural bloat. As this generated equation entirely lacks physical interpretability, we omit it here for brevity.

\textbf{LLM-PDESR Discovery (Flawless Structural Recovery):}
\begin{align}
    u_t &= 0.0500 u_{xx} + 0.9998u(1 - u - 0.4989v) \\
    v_t &= 0.0108 v_{xx} - 0.2023(v u_x)_x + 0.7972v(u - 1)
\end{align}
\textbf{Analysis:} LLM-PDESR achieved a \textbf{flawless structural recovery} of this highly complex coupled system. The framework successfully deduced the intricate cross-chemotaxis dependency out of an exponentially vast multivariable search space (where combinations of $u, v, u_x, v_x, u_{xx}, v_{xx}$ combinatorially explode). Figure~\ref{fig:pde5}(c) presents the ultimate validation of this capability: at $t=10.0$, the discovered profiles for both the prey (dashed blue line) and the tracking predator (dashed red line) perfectly overlay the exact ground truth across the entire spatial domain. The explicit localization of the predator's density peaks precisely where the prey gradients are steepest proves that the LLM is genuinely executing multivariable symbolic reasoning, rather than merely fitting isolated geometric shapes.

\begin{figure}[htbp]
    \centering
    \includegraphics[width=\textwidth]{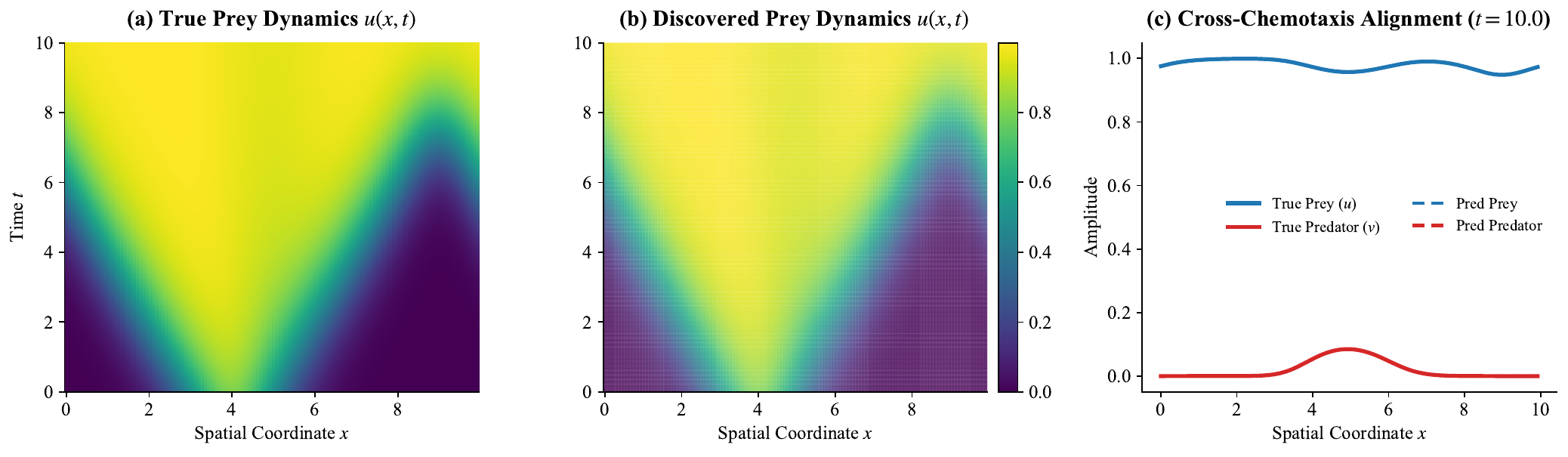}
    \caption{\textbf{Cross-Chemotactic Predator-Prey.} LLM-PDESR perfectly isolates the cross-variable spatial coupling $(vu_x)_x$. The microscopic alignment in (c) visually confirms that the discovered equation faithfully forces the predator ($v$) to actively track the spatial gradient of the prey ($u$), with zero observable deviation from the ground truth.}
    \label{fig:pde5}
\end{figure}

\section{Physical interpretation and structural generalization of the 1D-CACE surrogate}
\label{Appendix:F}

In the main text, we introduced the 1D Coupled Atmospheric Circulation Equation (1D-CACE) extracted by LLM-PDESR directly from chaotic ERA5 observational data:
\begin{equation*}
\begin{aligned}
    u_t &= \theta_1 u_{xx} + \theta_2 u u_x + \theta_3 v u_x + \theta_4 v_x + \theta_5 u + \theta_6 v + \theta_7 x v \\
    v_t &= \theta_8 v_{xx} + \theta_9 v v_x + \theta_{10} u v_x + \theta_{11} u_x + \theta_{12} v + \theta_{13} u + \theta_{14} x u
\end{aligned}
\end{equation*}

Evaluating symbolic models on genuine atmospheric data presents a fundamental closure problem. To extract the 1D spatiotemporal fields $u(x,t)$ and $v(x,t)$ from the raw 2D geographic bounding boxes, we apply a \textbf{meridional average}---specifically, computing the spatial mean across the specified latitude band for each longitude coordinate $x$. However, atmospheric turbulence remains inherently chaotic and is governed by latent 3D thermodynamic variables (e.g., vertical convection, baroclinic instability, latent heat release) that are inevitably omitted from such a 1D kinematic extraction. Faced with this severe dimensionality truncation, traditional methods often fail catastrophically. Although applicable to multivariate systems, WSINDy\_PDE yields a severely overfitted, unphysical expression here, which we omit for brevity. Rather than claiming the discovery of a complete fundamental law, we rigorously demonstrate below how the discovered structural skeleton serves as a highly effective, physically consistent \textit{dynamical surrogate}. We detail how this framework compresses macroscopic 3D dynamics into an interpretable 1D manifold, and subsequently validate how this physical robustness ensures structural generalization across distinct spatiotemporal regimes.

\subsection{Effective dynamical mechanics and spatial closure}
The discovered structural skeleton aligns remarkably well with the principles of geophysical fluid dynamics (GFD), breaking down into several effective governing mechanisms:

\paragraph{Effective eddy viscosity and kinematic dissipation ($\theta_1 u_{xx}, \theta_8 v_{xx}$):} Analogous to the viscous terms in the Navier--Stokes equations, these second-order spatial derivatives function as effective sub-grid scale turbulent diffusion operators. For 10-meter wind fields, they mathematically parameterize the irreversible dissipation of kinetic energy and the smoothing of sharp velocity gradients induced by planetary boundary layer turbulence and surface friction.

\paragraph{Nonlinear advection and zonal-meridional momentum exchange ($\theta_2 u u_x, \theta_9 v v_x$ and $\theta_3 v u_x, \theta_{10} u v_x$):} While traditional 1D surrogate models (e.g., the 1D Burgers equation) are strictly limited to self-advection, LLM-PDESR identifies a deeply coupled advective system. Retaining the cross-advection terms ($\theta_3, \theta_{10}$) is critical; it captures the physical reality that the meridional wind ($v$) actively transports zonal momentum ($u$) and vice versa, maintaining physically plausible momentum balances within the planar formulation.

\paragraph{Linear coupling and implicit geopotential gradients ($\theta_4 v_x, \theta_{11} u_x$):} In simplified shallow water frameworks, spatial gradients of one velocity component often act as kinematic proxies for geopotential height perturbations that drive the orthogonal flow. Here, these terms serve as essential linear coupling mechanisms, providing the mathematical degrees of freedom necessary to approximate the propagation of inertial-gravity waves.

\paragraph{Rayleigh drag and background state relaxation ($\theta_5 u, \theta_6 v$ and $\theta_{12} v, \theta_{13} u$):} These zeroth-order terms function as linear damping mechanisms, acting analogously to Rayleigh friction (surface drag), which continuously dissipates kinetic energy and relaxes the velocity anomalies toward a local equilibrium. The cross-linear terms provide necessary adjustments to the steady-state balance between the orthogonal components.

\paragraph{First-order spatial closure for macroscopic forcings ($\theta_7 x v, \theta_{14} x u$):} 
From a strict fluid dynamics perspective, fundamental physical laws demand spatial translational symmetry. However, our model operates on a reduced 1D domain ($x$), rendering it inherently unable to explicitly resolve meridional variations (e.g., the latitude-dependent Coriolis parameter, $f \approx f_0 + \beta y$) or persistent 2D geographical forcings. 
Mathematically, when a spatially heterogeneous 2D/3D background state is projected onto a 1D manifold, the unresolved spatial variations can be optimally parameterized via a first-order Taylor expansion with respect to the spatial coordinate. Thus, the emergence of the explicit spatially dependent terms ($\theta_7 x v, \theta_{14} x u$) is not an overfitting artifact, but a rigorous \textit{first-order spatial closure}. These terms effectively break local translational symmetry to mathematically compensate for missing macroscopic background gradients (such as persistent zonal pressure asymmetries or the projected imprint of the $\beta$-effect). 

\subsection{Empirical validation via structural invariance}

The theoretical soundness of this spatial parameterization is precisely what enables the framework to transcend localized statistical memorization. In GFD, the nonlinear differential operators dictate the invariant physical rules (fluid motion), while the algebraic coefficients dictate the local environmental state (boundary conditions and forcings). To empirically validate this, we demonstrate robust generalization by \textit{strictly freezing the differential skeleton} and only re-optimizing the scalar coefficients $\theta$ for new domains. Note that fitting the temporal tendency ($\partial u/\partial t, \partial v/\partial t$) is orders of magnitude more challenging than predicting the state variables, as the derivative exponentially amplifies high-frequency chaotic noise.

\subsubsection{In-distribution baseline: North Pacific (Jan 1--10)}

The base regime spans the North Pacific ($25^\circ$--$45^\circ\text{N}$, $160^\circ$--$130^\circ\text{W}$), a dynamically active region during boreal winter characterized by the intense westerly jet stream. 

\textbf{Analysis:} As depicted in Figure~\ref{fig:era5_base}, the ground truth tendency exhibits severe high-frequency spatiotemporal turbulence. When parameterized for this region, the frozen 1D-CACE skeleton achieves an $R^2$ of \textbf{76.28\%}. The prediction closely tracks the dominant macroscopic Rossby wave manifolds while preserving meso-scale fluctuations, successfully avoiding the over-smoothing typical of deep neural networks. The unstructured residual error rigorously represents the latent influence of unresolvable 3D thermodynamic forcings.

\begin{figure}[htbp]
    \centering
    \includegraphics[width=\textwidth]{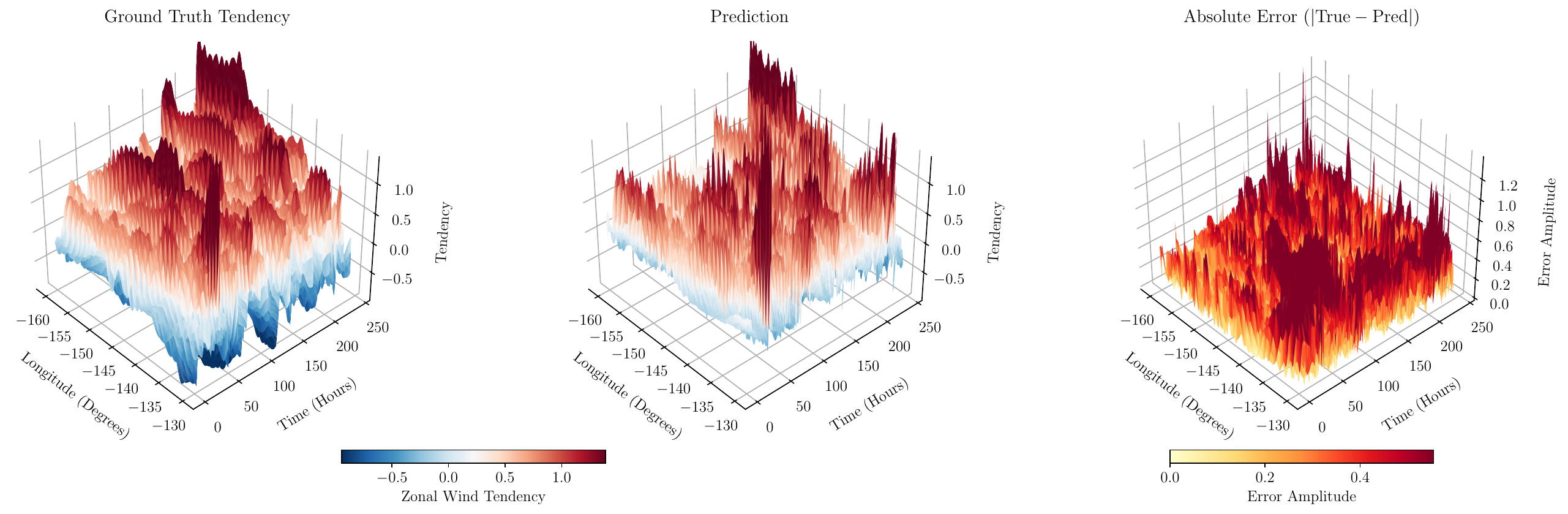} 
    \caption{\textbf{In-Distribution Performance (North Pacific, Jan 1--10).} The 1D-CACE skeleton reconstructs the chaotic wind tendency ($R^2=76.28\%$), successfully balancing macroscopic wave tracking with high-frequency fluctuations.}
    \label{fig:era5_base}
\end{figure}

\subsubsection{Spatial transferability: West Pacific (Jan 1--10)}

To rigorously test spatial robustness, we transfer the frozen skeleton to the West Pacific ($140^\circ$--$170^\circ\text{E}$). Dominated by the Kuroshio current's intense sea surface temperature gradients, this cyclogenesis zone presents boundary conditions fundamentally distinct from the Central Pacific. 

\textbf{Analysis:} Transferring a PDE to an unseen geographical domain typically risks numerical instability. However, simply by re-optimizing the scalar coefficients to absorb the new background state, our skeleton successfully reconstructs the dominant tendency manifolds ($R^2=$ \textbf{60.97\%}, Figure~\ref{fig:era5_spatial}). This stability confirms that the framework has successfully isolated the globally consistent advective-diffusive balance from localized geographical noise.

\begin{figure}[htbp]
    \centering
    \includegraphics[width=\textwidth]{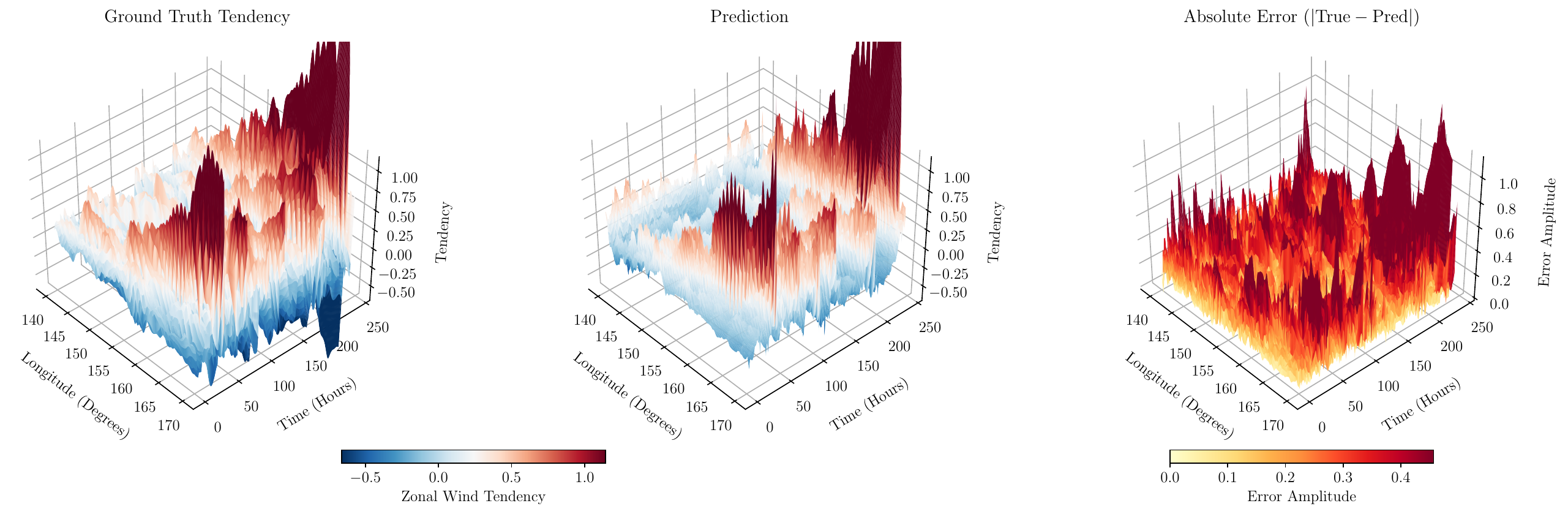} 
    \caption{\textbf{Spatial OOD Performance (West Pacific, Jan 1--10).} The frozen skeleton adapts robustly to distinct oceanic forcings without numerical divergence ($R^2=60.97\%$), confirming the structural invariance of the discovered advection terms.}
    \label{fig:era5_spatial}
\end{figure}

\subsubsection{Temporal phase shift: North Pacific (Feb 1--10)}

Finally, we evaluate the model under Temporal Drift by shifting the window forward by one month. Due to rapid internal atmospheric variability, the flow states and storm tracks in February represent a phase shift from early January. 

\textbf{Analysis:} Despite this heavily altered topology (Figure~\ref{fig:era5_temporal}), the 1D-CACE skeleton demonstrates robust temporal resilience ($R^2=$ \textbf{58.64\%}). While purely statistical models trained on narrow temporal windows often suffer severe degradation under such phase shifts, LLM-PDESR responds to the invariant physical state gradients rather than sequentially memorized patterns. This confirms that the framework extracts genuine dynamical abstractions governing mid-latitude transport, effectively unifying mathematical interpretability with real-world predictive robustness.

\begin{figure}[htbp]
    \centering
    \includegraphics[width=\textwidth]{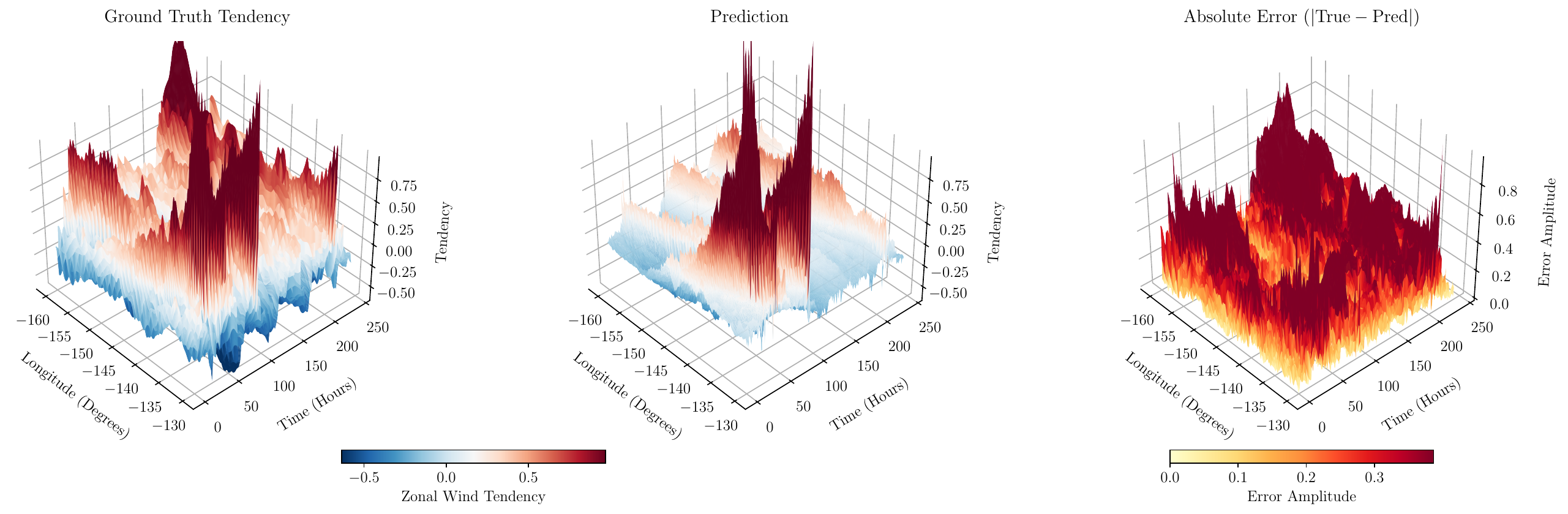} 
    \caption{\textbf{Temporal Drift Performance (North Pacific, Feb 1--10).} Despite severe atmospheric phase shifts, the invariant skeleton reconstructs new weather tendencies ($R^2=58.64\%$), proving the extraction of genuine dynamical abstractions over statistical memorization.}
    \label{fig:era5_temporal}
\end{figure}

\section{Limitations}
\label{Appendix:Limitations}

While LLM-PDESR demonstrates robust symbolic discovery capabilities across diverse physical systems, several limitations warrant future investigation. First, as analyzed in the traffic flow failure case, the framework struggles with PDEs exhibiting extreme numerical stiffness or near-discontinuous shock fronts. Such steep spatial gradients challenge the SWR integration step and disrupt the gradient-based BFGS parameter optimization. Second, the reliance on iterative LLM inference within the Pareto evolutionary loop introduces a significant computational overhead compared to traditional sparse regression algorithms, making it highly dependent on GPU resources and inference latencies. Finally, the current numerical evaluation pipeline is built upon 1D quintic B-splines. Extending this high-order continuous formulation to 2D or 3D dynamical systems on unstructured grids introduces non-trivial topological complexities. Addressing these constraints—potentially through adaptive mesh refinement and deploying smaller, domain-distilled language models—remains an active area for our future work.

\section{Broader Impacts}
\label{Appendix:BroaderImpacts}

While our work focuses on foundational algorithmic advancements in scientific machine learning, we acknowledge its broader impacts here. On a positive note, automating PDE discovery can significantly accelerate research in critical societal domains like climate modeling (as demonstrated with our ERA5 application). Conversely, deploying LLM-driven scientific tools in high-stakes engineering systems without proper human oversight poses risks, as mathematically deceptive but physically flawed hypotheses could mislead downstream applications.


\end{document}